%% file: main.tex
\documentclass{article}
\PassOptionsToPackage{numbers, compress}{natbib}

\usepackage[final]{neurips_2019}

\usepackage[utf8]{inputenc} 
\usepackage[T1]{fontenc}    
\usepackage{hyperref}       
\usepackage{url}            
\usepackage{booktabs}       
\usepackage{amsfonts}       
\usepackage{nicefrac}       
\usepackage{microtype}      
\usepackage{amssymb,amsmath,bbm,mathrsfs,amsthm,bm,enumitem,csquotes,mathtools,bookmark,multicol,pgfplots}
\tikzset{mark options={mark size=2, line width=1pt}}
\pgfplotsset{ticks=none}
\pgfplotsset{try min ticks=3}
\pgfplotsset{max space between ticks=50}
\usetikzlibrary{3d}
\usetikzlibrary{shadings}
\usepgfplotslibrary{groupplots}
\usepgfplotslibrary{patchplots}
\pgfplotsset{colormap/hot}

\makeatletter
\tikzoption{canvas is plane}[]{\@setOxy#1}
\def\@setOxy O(#1,#2,#3)x(#4,#5,#6)y(#7,#8,#9)%
  {\def\tikz@plane@origin{\pgfpointxyz{#1}{#2}{#3}}%
   \def\tikz@plane@x{\pgfpointxyz{#4}{#5}{#6}}%
   \def\tikz@plane@y{\pgfpointxyz{#7}{#8}{#9}}%
   \tikz@canvas@is@plane
  }
\makeatother

\newtheorem{theorem}{Theorem}[section]
\newtheorem{corollary}[theorem]{Corollary}
\newtheorem{lemma}[theorem]{Lemma}
\newtheorem{proposition}[theorem]{Proposition}
\newtheorem{example}[theorem]{Example}
\newtheorem{definition}[theorem]{Definition}
\newcommand{\R}{\mathbb{R}}
\newcommand{\N}{\mathbb{N}}
\renewcommand{\L}{\mathcal{L}}
\newcommand{\cN}{\mathcal{N}}
\newcommand{\cP}{\mathcal{P}}
\newcommand{\cR}{\mathcal{R}}
\newcommand{\cB}{\mathcal{B}}
\newcommand{\cA}{\mathcal{A}_L}
\newcommand{\W}{W^{1,\infty}}
\newcommand{\oneto}[1]{[{#1}]}
\newcommand{\esssup}[1]{\operatorname*{ess\ sup}\limits_{#1}}
\newcommand{\eps}{\varepsilon}

\newcommand{\<}{\langle}
\renewcommand{\>}{\rangle}
\providecommand{\1}{{\ensuremath{\mathbf{1}}}}
\newcommand{\m}{\oneto{m}}
\DeclareMathOperator{\sgn}{sgn}
\newcommand{\ReLU}{\rho}
\newcommand{\notimplies}{%
\mathrel{{\ooalign{\hidewidth$\not\phantom{=}$\hidewidth\cr$\implies$}}}}
\DeclareMathOperator*{\infp}{\text{inf\vphantom{p}}}
\DeclareMathOperator*{\supp}{\text{sup}}

\title{How degenerate is the parametrization of neural networks with the {ReLU} activation function?}

\author{
   Julius Berner \\
   Faculty of Mathematics, University of Vienna \\
   Oskar-Morgenstern-Platz 1, 1090 Vienna, Austria\\
   \texttt{julius.berner@univie.ac.at}
   \And
   Dennis Elbr\"achter \\
   Faculty of Mathematics, University of Vienna \\
   Oskar-Morgenstern-Platz 1, 1090 Vienna, Austria\\
   \texttt{dennis.elbraechter@univie.ac.at} \\
   \AND
   Philipp Grohs \\
Faculty of Mathematics and Research Platform DataScience@UniVienna,
University of Vienna \\ 
Oskar-Morgenstern-Platz 1, 1090 Vienna, Austria\\
   \texttt{philipp.grohs@univie.ac.at} \\
}

\begin{document}

\maketitle

\begin{abstract}
Neural network training is usually accomplished by solving a non-convex optimization problem using stochastic gradient descent. Although one optimizes over the networks parameters, the main loss function generally only depends on the realization of the neural network, i.e.\@ the function it computes. Studying the optimization problem over the space of realizations opens up new ways to understand neural network training. In particular, usual loss functions like mean squared error and categorical cross entropy are convex on spaces of neural network realizations, which themselves are non-convex. Approximation capabilities of neural networks can be used to deal with the latter non-convexity, which allows us to establish that for sufficiently large networks local minima of a regularized optimization problem on the realization space are almost optimal. Note, however, that each realization has many different, possibly degenerate, parametrizations. In particular, a local minimum in the parametrization space needs not correspond to a local minimum in the realization space. To establish such a connection, inverse stability of the realization map is required, meaning that proximity of realizations must imply proximity of corresponding parametrizations. We present pathologies which prevent inverse stability in general, and, for shallow networks, proceed to establish a restricted space of parametrizations on which we have inverse stability w.r.t.\@ to a Sobolev norm. Furthermore, we show that by optimizing over such restricted sets, it is still possible to learn any function which can be learned by optimization over unrestricted sets.
\end{abstract}

\section{Introduction and Motivation}\label{sec:Sec1}
In recent years much effort has been invested into explaining and understanding the overwhelming success of deep learning based methods. On the theoretical side, impressive approximation capabilities of neural networks have been established~\cite{bolcskei2017optimal,Burger2001235,Funahashi1989183,guhring2019error,perekrestenko2018universal,Petersen2017,ShaCC2015provableAppDNN,yarotsky2017error}. No less important are recent results on the generalization of neural networks, which deal with the question of how well networks, trained on limited samples, perform on unseen data~\cite{anthony2009,arora2018stronger,BartlettFT17,Bartlett2017Nearly-tightNetworks,Berner2018AnalysisEquations,Golowich17,neyshabur2017exploring}. Last but not least, the optimization error, which quantifies how well a neural network can be trained by applying stochastic gradient descent to an optimization problem, has been analyzed in different scenarios~\cite{Allen-Zhu2018AOver-Parameterization,choromanska2015loss,Du2018GradientNetworks,kawaguchi2016deep,li2018learning,li2017convergence,mei2018mean,Shamir2013}. While there are many interesting approaches to the latter question, they tend to require very strong assumptions (e.g.\@ (almost) linearity, convexity, or extreme over-parametrization). Thus a satisfying explanation for the success of stochastic gradient descent for a non-smooth, non-convex problem remains elusive.\\
In the present paper we intend to pave the way for a functional perspective on the optimization problem. This allows for new mathematical approaches towards understanding the training of neural networks, some of which are demonstrated in Section~\ref{sec:opt}. To this end we examine degenerate parametrizations with undesirable properties in Section~\ref{sec:degen}. These can be roughly classified as 
\begin{enumerate}[label=C.\arabic*]
    \item unbalanced magnitudes of the parameters
    \item weight vectors with the same direction
    \item weight vectors with directly opposite directions.
\end{enumerate}
Under conditions designed to avoid these degeneracies, Theorem~\ref{thm:main} establishes inverse stability for shallow networks with ReLU activation function. This is accomplished by a refined analysis of the behavior of ReLU networks near a discontinuity of their derivative.
Proposition~\ref{prop:local} shows how inverse stability connects the loss surface of the parametrized minimization problem to the loss surface of the realization space problem.
In Theorem~\ref{thm:optPara} we showcase a novel result on almost optimality of local minima of the parametrized problem obtained by analyzing the realization space problem.
Note that this approach of analyzing the loss surface is conceptually different from previous approaches as in~\cite{choromanska2015loss,goodfellow2014qualitatively,li2018visualizing,Nguyen17loss,Pennington2017loss,safran2016quality}.

\subsection{Inverse Stability of Neural Networks}
We will focus on neural networks with the ReLU activation function $\rho(x):=x_+$, and adapt the mathematically convenient notation from~\cite{Petersen2017}, which distinguishes between the \emph{parametrization} of a neural network and its \emph{realization}.
Let us define the set $\cA$ of all network \emph{architectures} with depth $L\in\N$, input dimension $d\in\N$, and output dimension $D\in\N$ by
\begin{equation}
    \cA:= \{ (N_0,\dots,N_L)\in\N^{L+1} \colon N_0=d, N_L=D \}.
\end{equation}
The architecture $N\in\cA$ simply specifies the number of neurons $N_l$ in each of the $L$ layers. We can then define the space $\cP_N$ of \emph{parametrizations} with architecture $N\in\cA$ as
\begin{align} \label{eq:nn_set}
    \cP_N:=\prod_{\ell=1}^L \left( \R^{N_\ell\times N_{\ell-1}}\times\R^{N_\ell} \right), 
\end{align}
the set $\cP=\bigcup_{N\in\cA} \cP_N$ of all parametrizations with architecture in $\cA$, and the \emph{realization} map  
\begin{align}\begin{split} \label{eq:real_map}
    \cR\colon\cP&\to C(\R^{d},\R^{D}) \\ 
    \Theta=((A_\ell,b_\ell))_{\ell=1}^L&\mapsto\cR(\Theta):=W_L\circ \rho\circ W_{L-1}\dots\rho\circ W_1,
\end{split}\end{align}
where $W_\ell (x):=A_\ell x + b_\ell$ 
and $\rho$ is applied component-wise.
We refer to $A_\ell$ and $b_\ell$ as the weights and biases in the $\ell$-th layer.\\
Note that a parametrization $\Theta\in\Omega\subseteq\cP$ uniquely induces a realization $\cR(\Theta)$ in the realization space $\cR(\Omega)$, while in general there can be multiple non-trivially different parametrizations with the same realization. To put it in mathematical terms, the realization map is not injective.
Consider the basic counterexample 
\begin{equation}
    \Theta=\big( (A_1,b_1),\dots, (A_{L-1},b_{L-1}), (0,0) \big) \quad \text{and} \quad \Gamma=\big( (B_1,c_1),\dots, (B_{L-1},c_{L-1}), (0,0) \big)
\end{equation}
from~\cite{Petersen2018TopologicalSize} where regardless of $A_\ell,B_\ell,b_\ell$ and $c_\ell$ 
both realizations coincide with 
$\cR(\Theta)=\cR(\Gamma)=0$.
However, it it is well-known that the realization map is locally Lipschitz continuous, meaning that close\footnote{On the finite dimensional vector space $\cP_N$ all norms are equivalent and we take w.l.o.g.\@ the maximum norm $\|\Theta\|_\infty$, i.e.\@ the maximum of the absolute values of the entries of the $A_\ell$ and $b_\ell$.}
parametrizations in $\cP_N$ induce realizations which are close in the uniform norm on compact sets, see e.g.\@~\cite[Lemma 14.6]{anthony2009},~\cite[Theorem 4.2]{Berner2018AnalysisEquations}, and~\cite[Proposition 5.1]{Petersen2018TopologicalSize}.\\
We will shed light upon the inverse question. Given realizations $\cR(\Gamma)$ and $\cR(\Theta)$ that are close, do the parametrizations $\Gamma$ and $\Theta$ have to be close? 
In an abstract setting we measure the proximity of realizations in the norm $\|\cdot\|$ of a Banach space $\cB$ with $\cR(\cP)\subseteq \cB$, while concrete Banach spaces of interest will be specified later.
In view of the above 
counterexample we will, at the very least, need to allow for the reparametrization of one of the networks, i.e.\@ we arrive at the following question.
\begin{quote}
Given $\cR(\Gamma)$ and $\cR(\Theta)$ 
that are close, does there exist a parametrization $\Phi$ with $\cR(\Phi)=\cR(\Theta)$ such that $\Gamma$ and $\Phi$ are close?
\end{quote}
As we will see in Section~\ref{sec:degen}, this question is fundamentally connected to understanding the redundancies and degeneracies of the way that neural networks are parametrized. By suitable regularization, i.e.\@ considering a subspace $\Omega\subseteq\cP_N$ of parametrizations, we can avoid these pathologies and establish a positive answer to the question above. For such a property the term
\emph{inverse stability} was introduced in \cite{Petersen2018TopologicalSize}, which constitutes the only other research conducted in this area, as far as we are aware.
\begin{definition}[Inverse stability]\label{def:InverseStability}
Let $s,\alpha>0$, $N\in\cA$, and $\Omega\subseteq\cP_N$.
We say that the realization map is $(s,\alpha)$ inverse stable on $\Omega$ w.r.t.\@ $\|\cdot\|$, if for all $\Gamma\in\Omega$ and $g\in\cR(\Omega)$ there exists $\Phi\in\Omega$ with
\begin{align}
\cR(\Phi)=g \quad \text{and} \quad \| \Phi- \Gamma\|_\infty \le s \|g-\cR(\Gamma)\|^\alpha.
\end{align}
\end{definition}
In Section~\ref{sec:degen} we will see why inverse stability fails w.r.t.\@ the uniform norm.
Therefore, we consider a norm which takes into account not only the maximum error of the function values but also of the gradients. In mathematical terms, we make use of the Sobolev norm $\|\cdot\|_{W^{1,\infty}(U)}$ (on some domain $U\subseteq \R^{d}$) defined for every (locally) Lipschitz continuous function $g\colon\R^d\to\R^D$ by $\|g\|_{W^{1,\infty}(U)}:=\max\{\|g\|_{L^\infty(U)},|g|_{\W(U)}\}$
with the Sobolev semi-norm ${|\cdot|_{\W(U)}}$ given by
\begin{align}
   |g|_{\W(U)}:=\|Dg\|_{L^\infty(U)}=\esssup{x\in U}\|Dg(x)\|_\infty.
\end{align}
See~\cite{Evans2015MeasureEdition} for further information on Sobolev norms, and~\cite{berner2019towards} for further information on the derivative of ReLU networks.
\subsection{Implications of inverse stability for neural network optimization} \label{sec:opt}
We proceed by demonstrating how inverse stability opens up new perspectives on the optimization problem which arises in neural network training. 
Specifically, consider a loss function $\mathcal{L}\colon C(\R^{d},\R^{D}) \to[0,\infty)$ on the space of continuous functions. 
For illustration, we take the commonly used mean squared error (MSE) which, for training data $((x^i,y^i))_{i=1}^n \in (\R^d \times \R^D)^n$, is given by
\begin{equation}\label{eq:Lossfunction}
    \mathcal{L}(g)=\tfrac{1}{n}\sum_{i=1}^n \|g(x^i)-y^i\|_2^2,\quad \text{for } g\in C(\R^d,\R^D).
\end{equation} 
Typically, the optimization problem is solved over some subspace of parametrizations $\Omega\subseteq\cP_N$, i.e.\@
  \begin{equation} \label{eq:min_par}
      \min_{\Gamma \in \Omega} \,  \mathcal{L}(\cR(\Gamma))=\min_{\Gamma \in \Omega} \, \tfrac{1}{n}\sum_{i=1}^n \|\cR(\Gamma)(x^i)-y^i\|_2^2.
  \end{equation}
From an abstract point of view, by writing $g=\cR(\Gamma)\in \cR(\Omega)$, this is equivalent to the corresponding optimization problem over the space of realizations $\cR(\Omega)$, i.e.\@
  \begin{equation} \label{eq:min_real}
      \min_{g \in \cR(\Omega)}  \mathcal{L}(g)=\min_{g \in \cR(\Omega)} \tfrac{1}{n}\sum_{i=1}^n \|g(x^i)-y^i\|_2^2.
  \end{equation}
However, the loss landscape of the optimization problem~\eqref{eq:min_par} is only properly connected to the loss landscape of the optimization problem~\eqref{eq:min_real} if the realization map is inverse stable on $\Omega$. Otherwise a realization $g\in\cR(\cP_N)$ can be arbitrarily close to a global minimum in the realization space but every parametrization $\Phi$ with $\cR(\Phi)=g$ is far away from the corresponding global minimum in the parametrization space.
Moreover, local minima of~\eqref{eq:min_par} in the parametrization space must correspond to local minima of~\eqref{eq:min_real} in the realization space if and only if we have inverse stability.
\begin{proposition}[Parametrization minimum $\Rightarrow$ realization minimum]
\label{prop:local}
Let $N\in\cA$, $\Omega\subseteq\cP_N$ and 
let the realization map be $(s,\alpha)$ inverse stable on $\Omega$ w.r.t.\@ $\|\cdot\|$.
Let $\Gamma_*\in\Omega$ be a local minimum of $\L\circ\cR$ on $\Omega$ with radius $r > 0$, i.e.\@ for all $\Phi\in \Omega$ with $\|\Phi-\Gamma_*\|_\infty\leq r$ it holds that 
\begin{equation}
    \L(\cR(\Gamma_*))\leq \L(\cR(\Phi)).
\end{equation}
Then $\cR(\Gamma_*)$ is a local minimum of $\L$ on $\cR(\Omega)$ with radius $(\frac{r}{s})^{1/\alpha}$, i.e.\@ for all $g\in \cR(\Omega)$ with $\|g-\cR(\Gamma_*)\|\leq (\frac{r}{s})^{1/\alpha}$ it holds that 
\begin{equation}
    \L(\cR(\Gamma_*))\leq \L(g).
\end{equation}
\end{proposition}
See Appendix~\ref{sec:AppProof1} for a proof and Example~\ref{ex:loc} for a counterexample in the case that inverse stability is not given.
Note that in~\eqref{eq:min_real} we consider a problem with convex loss function but non-convex feasible set, see~\cite[Section 3.2]{Petersen2018TopologicalSize}. This opens up new avenues of investigation using tools from functional analysis and allows utilizing recent results~\cite{Gribonval2019ApproximationNetworks,Petersen2018TopologicalSize} exploring the topological properties of neural network realization spaces.\\
As a concrete demonstration we provide with Theorem~\ref{thm:opt} a strong result obtained on the realization space, which estimates the quality of a local minimum based on its radius and the approximation capabilities of the chosen architecture for a class of functions $S$.
Specifically let $C>0$, let
${\Lambda\colon \cB\to [0,\infty)}$ be a quasi-convex regularizer, and define
\begin{equation}\label{eq:def_S}
    S:=\{f\in \cB\colon \Lambda(f)\le C\}.
\end{equation}
 We denote the sets of regularized parametrizations by \begin{equation}
    \Omega_N:=\{\Phi\in\cP_N\colon \Lambda(\cR(\Phi))\leq C\}
\end{equation} 
and assume that the loss function $\L$ is convex and $c$-Lipschitz continuous on $S$. Note that virtually all relevant loss functions are convex and locally Lipschitz continuous on $C(\R^{d},\R^{D})$.
Employing Proposition~\ref{prop:local}, inverse stability can then be used to derive the following result for the practically relevant parametrized problem, showing that for sufficiently large architectures local minima of a regularized neural network optimization problem are almost optimal.
\begin{theorem}[Almost optimality of local parameter minima] \label{thm:optPara}
Assume that $S$ is compact in the $\|\cdot\|$-closure of $\cR(\cP)$ and that for every $N\in\cA$ the realization map is $(s,\alpha)$ inverse stable on $\Omega_N$ w.r.t.\@ $\|\cdot\|$\,.
Then for all $\eps,r>0$
there exists $n(\eps,r)\in\cA$
such that for every $N\in\cA$ with $N_1\ge n_1(\eps,r),\dots,N_{L-1}\ge n_{L-1}(\eps,r)$ the following holds: \\
Every local minimum $\Gamma_*$ with radius at least $r$ of $\min_{\Gamma\in\Omega_N}\L(\cR(\Gamma))$ satisfies
\begin{equation}
    \L(\cR(\Gamma_*))\leq\min_{\Gamma\in\Omega_N}\L(\cR(\Gamma))+\eps.
\end{equation}
\end{theorem}
See Appendix~\ref{sec:AppProof1} for a proof and note that here it is important to have an inverse stability result, where the parameters $(s,\alpha)$ do not depend on the size of the architecture, which we achieve for $L=2$ and $\cB=W^{1,\infty}$.
Suitable $\Lambda$ would be Besov norms which constitute a common regularizer in image and signal processing.
Moreover, note that the required size of the architecture in Theorem~\ref{thm:optPara} can be quantified, if one has approximation rates for $S$. In particular, this approach allows the use of approximation results in order to explain the success of neural network optimization and enables a combined study of these two aspects, which, to the best of our knowledge, has not been done before. 
Unlike in recent literature, our result needs no assumptions on the sample set (incorporated in the loss function, see \eqref{eq:Lossfunction}), in particular we do not require \enquote{overparametrization} with respect to the sample size. Here the required size of the architecture only depends on the complexity of $S$, i.e.\@ the class of functions one wants to approximate, the radius of the local minima of interest, the Lipschitz constant of the loss function, and the parameters of the inverse stability.\\
In the following we restrict ourselves to two-layer ReLU networks without biases, where we present a proof for $(4,1/2)$ inverse stability w.r.t.\@ the Sobolev semi-norm on a suitably regularized space of parametrizations. Both the regularizations as well as the stronger norm (compared to the uniform norm) will shown to be necessary in Section~\ref{sec:degen}. 
We now present, in an informal way, a collection of our main results. A short proof making the connection to the formal results can be found in Appendix~\ref{sec:AppProof1}.
\begin{corollary}[Inverse stability and implications - colloquial]
\label{cor:inform}
Suppose we are given data $((x^i,y^i))_{i=1}^n \in (\R^d \times \R^D)^n$ and want to solve a typical minimization problem for ReLU networks with shallow architecture $N=(d,N_1,D)$, i.e.\@
\begin{equation} \label{eq:opt_typ}
\min_{\Gamma \in \cP_N} \, \tfrac{1}{n}\sum_{i=1}^n \|\cR(\Gamma)(x^i)-y^i)\|_2^2.
\end{equation}
First we augment the architecture to $\tilde{N}=(d+2,N_1+1,D)$, while omitting the biases, and augment the samples to $\tilde{x}^i=(x^i_1,\dots,x^i_d,1,-1)$. Moreover, we assume that the parametrizations 
\begin{equation}
  \Phi=\left(\big( [ a_1 | \dots | a_{N_1+1} ]^T , 0 \big),([ c_1 | \dots | c_{N_1+1} ],0) \right) \in \Omega \subseteq \cP_{\tilde{N}}
\end{equation}
are regularized such that 
\begin{enumerate}[label=C.\arabic*]
    \item the network is balanced, i.e.\@ $\|a_i\|_\infty=\|c_i\|_\infty$,
    \item no non-zero weight vectors in the first layer are redundant, i.e.\@ $a_i \not\parallel a_j$,
    \item \label{it:C3cor} the last two coordinates of each weight vector $a_i$ are strictly positive.
\end{enumerate}
Then for the new minimization problem
\begin{equation}\label{eq:newmini}
    \min_{\Phi \in \Omega} \, \tfrac{1}{n}\sum_{i=1}^n  \|\cR(\Phi)(\tilde{x}^i)-y^i\|_2^2
    \end{equation}
the following holds:
\begin{enumerate}
    \item \label{it:col_first} If $\Phi_*$ is a local minimum of \eqref{eq:newmini} with radius $r$, then $\cR(\Phi_*)$ is a local minimum of $\min_{g \in \cR(\Omega)} \, \tfrac{1}{n}\sum_{i=1}^n  \|g(\tilde{x}^i)-y^i\|_2^2$ with radius at least $\tfrac{r^2}{16}$ w.r.t.\@ $|\cdot|_{\W}$.
    \item \label{it:col_sec} The global minimum of~\eqref{eq:newmini} is at least as good as the global minimum of~\eqref{eq:opt_typ}, i.e.\@
    \begin{equation}\label{augmentation}
    \min_{\Phi \in \Omega} \, \tfrac{1}{n}\sum_{i=1}^n \|\cR(\Phi)(\tilde{x}^i)-y^i\|_2^2\le \min_{\Gamma \in \cP_{N}} \, \tfrac{1}{n}\sum_{i=1}^n \|\cR(\Gamma)(x^i)-y^i\|_2^2.
    \end{equation}
    \item By further regularizing~\eqref{eq:newmini} in the sense of Theorem~\ref{thm:optPara}, we can estimate the quality of its local minima.
\end{enumerate}
\end{corollary}

This argument is not limited to the MSE loss function but works for any loss function based on evaluating the realization. 
The omission of bias weights is standard in neural network optimization literature~\cite{choromanska2015loss,Du2018GradientNetworks,kawaguchi2016deep,li2018learning}. While this severely limits the functions that can be realized with a given architecture, it is sufficient to augment the problem by one dimension in order to recover the full range of functions that can be learned~\cite{Allen-Zhu2018AOver-Parameterization}. 
Here we augment by two dimensions, so that the third regularization condition~\ref{it:C3cor} can be fulfilled without loosing range. 
Moreover, note that, for simplicity of presentation, the regularization assumptions stated above are stricter than necessary and possible relaxations are discussed in Section~\ref{sec:main}.
\section{Obstacles to inverse stability - degeneracies of ReLU parametrizations}
\label{sec:degen}
In the remainder of this paper we focus on shallow ReLU networks without biases and define the corresponding space of parametrizations with architecture $N=(d,m,D)$ as ${\cN_N:=\R^{m\times d}\times\R^{D\times m}}$.
The realization map\footnote{This is a slight abuse of notation, justified by the the fact that $\cR$ acts the same on $\cP_N$ with zero biases $b_1,b_2$ and weights $A_1=A$ and $A_2=C$.} $\cR$ is, for every $\Theta=(A,C)=\big( [a_1 |\dots |a_m ]^T, [ c_1 | \dots | c_m ] \big) \in\cN_N$, given by
\begin{equation}
   \R^d \ni x \mapsto \cR(\Theta)(x)=C\rho(Ax)=\sum_{i=1}^m c_i \rho (\langle a_i , x \, \rangle).
\end{equation}
Note that each function $x\mapsto c_i \rho (\< a_i , x \>)$ represents a so-called ridge function which is zero on the half-space $\{x\in\R^d \colon \< a_i , x \> \le 0 \}$ and linear with constant derivative $c_i a_i^T\in\R^D\times\R^d$ on the other half-space. Thus, the $a_i$ are the normal vectors of the separating hyperplanes $\{x\in\R^d \colon \< a_i , x \> = 0 \}$ and consequently we refer to the weight vectors $a_i$ also as the directions of $\Theta$. Moreover, for $\Theta\in\cN_N$ it holds that $\cR(\Theta)(0)=0$ and, as long as the domain of interest $U\subseteq \R^d$ contains the origin, the Sobolev norm $\|\cdot\|_{\W(U)}$ is equivalent to its semi-norm, since
\begin{equation}
    \|\cR(\Theta)\|_{L^\infty(U)} \le
    \sqrt{d} \, \operatorname{diam}(U)  |\cR(\Theta)|_{\W},
\end{equation}
see also inequalities of Poincaré-Friedrichs type~\cite[Subsection 5.8.1]{Evans2010PartialEdition}.
Therefore, in the rest of the paper we will only consider the Sobolev semi-norm\footnote{For $m\in\N$ we abbreviate $\oneto{m}:=\{1,\dots,m\}.$}
\begin{equation} \label{eq:s_semi}
    |\cR(\Theta)|_{\W(U)}=\esssup{x\in U} \Big\| \, \sum_{i\in[m] \colon \< a_i , x \> > 0} c_i a_i^T \, \Big\|_\infty.
\end{equation}
In~\eqref{eq:s_semi} one can see that in our setting
$|\cdot|_{\W(U)}$ 
is independent of $U$ (as long as $U$ contains a neighbourhood of the origin) and will thus be abbreviated by $|\cdot|_{\W}$. 

\subsection{Failure of inverse stability w.r.t. uniform norm}
All proofs for this section can be found in Appendix~\ref{sec:AppProof2}.
We start by showing that inverse stability fails w.r.t.\@ the uniform norm. This example is adapted from~\cite[Theorem 5.2]{Petersen2018TopologicalSize} and represents, to the best of our knowledge, the only degeneracy which has already been observed before.
  \begin{example}[Failure due to exploding gradient]\label{Ex0}
  Let 
  $\Gamma:=(0,0)\in\cN_{(2,2,1)}$ and $g_k\in\cR(\cN_{(2,2,1)})$ be given by
  (see Figure~\ref{fig:inv})
  \begin{equation}
      g_k(x) := k\rho(\< (k,0),x\>) - k \rho(\< (k,-\tfrac{1}{k^2}),x\>), \quad k\in\N.
  \end{equation}
  Then for every sequence $(\Phi_k)_{k\in\N}\subseteq\cN_{(2,2,1)}$ with $\cR(\Phi_k)=g_k$ it holds that
  \begin{align}
      \lim_{k\to\infty} \| \cR(\Phi_k) -\cR(\Gamma)\|_{L^\infty((-1,1)^2)} = 0 \quad \text{and} \quad \lim_{k\to\infty} \|\Phi_k-\Gamma\|_\infty=\infty.
  \end{align}
  \end{example}
\begin{figure}[!tb]
\centering
\input{inverse_tikz.tex}
\caption{The figure shows $g_k$ for $k=1,2$.}
\label{fig:inv}
\end{figure}
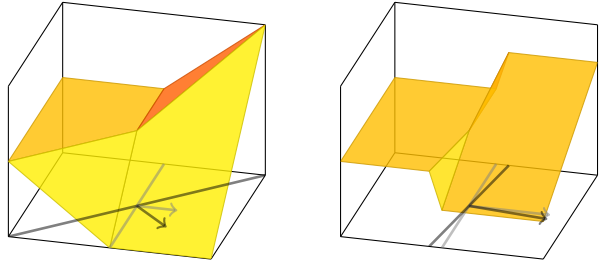
In particular, note that inverse stability fails here even for a non-degenerate parametrization of the zero function $\Gamma=(0,0)$.
However, for this type of counterexample the magnitude of the gradient of $\cR(\Phi_k)$ needs to go to infinity, which is our motivation for looking at inverse stability w.r.t.\@ $|\cdot|_{\W}$.
  
\subsection{Failure of inverse stability w.r.t. Sobolev norm}\label{sec:SobolevPath}
In this section we present four degenerate cases where inverse stability fails w.r.t.\@ $|\cdot|_{\W}$. This collection of counterexamples is complete in the sense that we can establish inverse stability under assumptions which are designed to exclude these four pathologies.
\begin{example}[Failure due to complete unbalancedness]\label{Ex1}
Let $r>0$, $\Gamma:=\big((r,0),0\big)\in\cN_{(2,1,1)}$
and $g_k\in\cR(\cN_{(2,1,1)})$ be given by (see Figure~\ref{fig:bal})
\begin{align}
    g_k(x)=\tfrac{1}{k}\rho(\<(0,1),x\>), \quad k\in\N.
\end{align}
Then for every $k\in\N$ and $\Phi_k\in\cN_{(2,1,1)}$ with $\cR(\Phi_k)=g_k$ it holds that 
\begin{equation}
    |\cR(\Phi_k)-\cR(\Gamma)|_{\W}=\tfrac{1}{k}
    \quad \text{and} \quad  
    \|\Phi_k-\Gamma\|_\infty \geq r.
\end{equation}
\end{example}
This is a very simple example of a degenerate parametrization of the zero function, since $\cR(\Gamma)=0$ regardless of choice of $r$. The issue here is that we can have a weight pair, i.e.\@ $((r,0),0)$, where the product is independent of the value of one of the parameters. Note that in Example~\ref{Ex1b} one can see a slightly more subtle version of this pathology by considering $\Gamma_k:=\big((k,0),\tfrac{1}{k^2}\big)\in\cN_{(2,1,1)}$ instead. In that case one could still get an inverse stability estimate for each fixed $k$; the parameters of inverse stability $(s,\alpha)$ would however deteriorate with increasing $k$. In particular this demonstrates the need for some sort of balancedness of the parametrization, i.e.\@ control over $\|c_i\|_\infty$ and $\|a_i\|_\infty$ individually relative to $\|c_i\|_\infty \|a_i\|_\infty$.\\
Inverse stability is also prevented by redundant directions as the following example illustrates.
\begin{example}[Failure due to redundant directions]\label{Ex2}
Let 
\begin{align}
     \Gamma:= \left( \begin{bmatrix} 1 & 0 \\ 1 & 0 \end{bmatrix}, (1,1) \right) \in  \cN_{(2,2,1)}
\end{align}
and $g_k\in\cR(\cN_{(2,2,1)})$ be given by (see Figure~\ref{fig:red})
\begin{align}
    g_k(x):=2\rho(\<(1,0),x\>)+\tfrac{1}{k}\rho(\<(0,1),x\>), \quad k\in\N.
\end{align}
Then for every $k\in\N$ and $\Phi_k\in\cN_{(2,2,1)}$ with $\cR(\Phi_k)=g_k$ it holds that
\begin{equation}
   |\cR(\Phi_k)-\cR(\Gamma)|_{\W}=\tfrac{1}{k}
   \quad \text{and} \quad  
   \|\Phi_k-\Gamma\|_\infty \geq  1.
\end{equation}
\end{example}
 \begin{figure}[!tb]
    \centering
    \begin{minipage}{.45\textwidth}
        \centering
        \input{bal_tikz.tex}
        \caption{Shows $\cR(\Gamma)$ ($r=0.5$) and $g_3$.}
        \label{fig:bal}
    \end{minipage}
    \hspace{2em}
    \begin{minipage}{0.45\textwidth}
        \centering
        \input{redundand_tikz.tex}
        \caption{Shows $\cR(\Gamma)$ and $g_2$.}
        \label{fig:red}
    \end{minipage}
\end{figure}
The next example shows that not only redundant weight vectors can cause issues, but also weight vectors of opposite direction, as they would allow for a (balanced) degenerate parametrization of the zero function. 
\begin{example}[Failure due to opposite weight vectors 1]\label{Ex3}
Let $a_i\in\R^d$, $i\in\m$, be pairwise linearly independent with $\|a_i\|_\infty=1$ and
$\sum_{i=1}^m a_i=0$. We define 
\begin{equation}
\Gamma:= \left( [ a_1 | \dots | a_m |  -a_1 | \dots | -a_m ]^T, \big(1 ,\dots,1, -1, \dots, -1 \big) \right) \in  \cN_{(d,2m,1)}.
\end{equation}
Now let $v\in\R^d$ with $\|v\|_\infty =1$ be linearly independent to each 
 $a_i$, $i\in[m]$, and let $g_k\in\cR(\cN_{(d,2m,1)})$ be given by (see Figure~\ref{fig:flip1})
 \begin{align}
     g_k(x)=\tfrac{1}{k}\rho(\<v,x\>), \quad k\in\N.
 \end{align}
Then there exists a constant $C>0$ such that for every $k\in\N$ and every $\Phi_k\in\cN_{(d,2m,1)}$ with $\cR(\Phi_k)=g_k$ it holds that
\begin{align}
   |\cR(\Phi_k)-\cR(\Gamma)|_{\W} = \tfrac{1}{k}
   \quad \text{and} \quad  
   \|\Phi_k-\Gamma\|_\infty \geq C.
\end{align}
\end{example}
Thus we will need an assumption which prevents each individual $\Gamma$ in our restricted set from having pairwise linearly dependent weight vectors, i.e.\@ coinciding hyperplanes of non-differentiability. This, however, does not suffice as is demonstrated by the next example, which shows that the relation between the hyperplanes of the two realizations matters.
\begin{example}[Failure due to opposite weight vectors 2]\label{Ex4}
We define the weight vectors
\begin{equation}
    a_1^k=(k , k , \tfrac{1}{k}), \quad  a_2^k=( -k , k , \tfrac{1}{k}),\quad a_3^k=(0 , -\sqrt2 k , \tfrac{1}{\sqrt2k}), \quad c^k=(k,k,\sqrt2 k) 
\end{equation}
and consider the parametrizations (see Figure~\ref{fig:flip2})
\begin{equation}
\Gamma_k:= \left( \big[ -a_1^k  \big| -a_2^k  \big| -a_3^k \big]^T, c^k \right) \in  \cN_{(3,3,1)}, \quad \Theta_k := \left( \big[ a_1^k  \big| a_2^k  \big| a_3^k
\big]^T, c^k \right) \in  \cN_{(3,3,1)}.
\end{equation}
Then for every $k\in\N$ and every $\Phi_k\in\cN_{(3,3,1)}$ with $\cR(\Phi_k)=\cR(\Theta_k)$ it holds that
\begin{equation}
   |\cR(\Phi_k)-\cR(\Gamma_k)|_{\W}=3
   \quad \text{and} \quad  
   \|\Phi_k-\Gamma_k\|_\infty \geq  k.
\end{equation}
\end{example}
 \begin{figure}[!tb]
    \centering
    \begin{minipage}{.45\textwidth}
        \centering
        \input{flip1_tikz.tex}
        \caption{Shows $\cR(\Gamma)$ and $g_3$ ($a_1=(1,-\tfrac{1}{2})$, $a_2=(-1,-\tfrac{1}{2})$, $a_3=(0,1)$, $v=(1,0)$).}
        \label{fig:flip1}
    \end{minipage}
    \hspace{2em}
    \begin{minipage}{0.45\textwidth}
        \centering
        \input{flip_tikz.tex}
        \caption{Shows the weight vectors 
          of $\Theta_2$ (grey) and $\Gamma_2$ (black).}
        \label{fig:flip2}
    \end{minipage}
\end{figure}
Note that $\Gamma$ and $\Theta$ need to have multiple exactly opposite weight vectors which add to something small (compared to the size of the individual vectors), but not zero, since otherwise reparametrization would be possible (see Lemma~\ref{lem:exact}).
\section{Inverse stability for two-layer {ReLU} Networks}\label{sec:main}
We now establish an inverse stability result using assumptions designed to exclude the pathologies from the previous section. First we present a rather technical theorem for output dimension one which considers a parametrization $\Gamma$ in the unrestricted parametrization space $\cN_N$ and a function $g$ in the the corresponding function space $\cR(\cN_N)$. The aim is to use assumptions which are as weak as possible, while allowing us to find a parametrization $\Phi$ of $g$, whose distance to $\Gamma$ can be bounded relative to $|g-\cR( \Gamma)|_{\W}$.
We then continue by defining a restricted parametrization space $\cN_N^*$, for which we get uniform inverse stability (meaning that we get the same estimate for every $\Gamma\in\cN_N^*$).
\begin{theorem}[Inverse stability at $\Gamma\in\cN_N$] \label{thm:main}
   Let $d,m\in\N$, $N:=(d,m,1)$, $\beta\in[0,\infty)$, let $\Gamma= \Big(\big[ a_1^\Gamma \big| \dots \big| a_m^\Gamma \big]^T, c^\Gamma \Big)\in\cN_{N}$, $g\in\cR(\cN_N)$, and let $I^\Gamma:=\{i\in\m\colon a^\Gamma_i\neq 0\}$.\\
   Assume that the following conditions are satisfied:
   \begin{enumerate}[label=C.\arabic*]
   \item\label{Cond1} 
   It holds for all $i\in\m$ with $\|c_i^\Gamma a_i^\Gamma\|_\infty\leq 2|g-\cR(\Gamma)|_{\W}$ that
   $|c_i^\Gamma|,\|a_i^\Gamma\|_\infty\leq \beta$. 
   \item\label{Cond2} 
   It holds for all $i,j\in I^{\Gamma}$ with $i\neq j$ that $\frac{a_j^\Gamma}{\|a_j^\Gamma\|_\infty}\neq\frac{a_i^\Gamma}{\|a_i^\Gamma\|_\infty}$.
   \item There exists a parametrization $\Theta= \Big(  \big[ a_1^\Theta \big| \dots \big| a_m^\Theta \big]^T, c^\Theta \Big)\in\cN_N$ such that $\cR(\Theta)=g$ and 
   \begin{enumerate}\label{Cond3}
    \item\label{Cond3b}  
    it holds for all $i,j\in I^\Gamma$ with $i\neq j$ that 
    $ 
    \frac{a_j^\Gamma}{\|a_j^\Gamma\|_\infty}\neq-\frac{a_i^\Gamma}{\|a_i^\Gamma\|_\infty}
    $ 
    and for all $i,j\in I^{\Theta}$ with $i\neq j$ that $\frac{a_j^\Theta}{\|a_j^\Theta\|_\infty}\neq-\frac{a_i^\Theta}{\|a_i^\Theta\|_\infty}$, 
    \item\label{Cond3a} 
    it holds for all $i\in I^{\Gamma}$, $j\in I^{\Theta}$ that
    $
        \frac{a_i^\Gamma}{\|a_i^\Gamma\|_\infty}\neq -\frac{a_j^\Theta}{\|a_j^\Theta\|_\infty}
    $
    \end{enumerate}
    where $I^\Theta:=\{i\in\m\colon a^\Theta_i\neq 0\}$.
   \end{enumerate}
   Then there exists a parametrization $\Phi\in\cN_N$ with 
   \begin{align}\label{main_result}
   \cR(\Phi)=g \quad \text{and} \quad \|\Phi-\Gamma\|_\infty\leq \beta+2|g-\cR(\Gamma)|^{\frac{1}{2}}_{\W}.
   \end{align}
\end{theorem}
The proof can be found in Appendix~\ref{sec:AppProof3}.
Note that each of the conditions in the theorem above corresponds directly to one of the pathologies in Section~\ref{sec:SobolevPath}. Condition~\ref{Cond1}, which deals with unbalancedness, only imposes an restriction on the weight pairs whose product is small compared to the distance of $\cR(\Gamma)$ and $g$. As can be guessed from Example~\ref{Ex1} and seen in the proof of Theorem~\ref{thm:main}, such a balancedness assumption is in fact only needed to deal with degenerate cases, where $\cR(\Gamma)$ and $g$ have parts with mismatching directions of negligible magnitude. Otherwise a matching reparametrization is always possible.
Note that a balanced $\Gamma$ (i.e.\@ $|c^\Gamma_i|=\|a^\Gamma_i\|_\infty$) satisfies Condition~\ref{Cond1} with $\beta= (2|g-\cR(\Gamma)|_{\W})^{1/2}$.\\
It is also possible to relax the balancedness assumption by only requiring $|c^\Gamma_i|$ and $\|\Gamma_i\|_\infty$ to be close to $\|c^\Gamma_i a^\Gamma_i\|^{1/2}_\infty$, which would still give a similar estimate but with a worse exponent. In order to see that requiring balancedness does not restrict the space of realizations, observe that the ReLU is positively homogeneous (i.e.\@ $\rho(\lambda x)=\lambda\rho(x)$ for all $\lambda\geq0$, $x\in\R$). Thus balancedness can always be achieved simply by rescaling.\\
Condition~\ref{Cond2} requires $\Gamma$ to have no redundant directions, the necessity of which is demonstrated by Example~\ref{Ex2}. 
Note that prohibiting redundant directions does not restrict the space of realizations, see~\eqref{eq:redundant} in the appendix for details.
From a practical point of view, enforcing this condition could be achieved by a regularization term using a barrier function. Alternatively on could employ a non-standard approach of combining such redundant neurons by changing one of them according to~\eqref{eq:redundant} and either setting the other one to zero or removing it entirely\footnote{This could be of interest in the design of dynamic network architectures~\cite{liu2018darts,MIIKKULAINEN2019293,zoph2018learning} and is also closely related to the co-adaption of neurons, to counteract which, dropout was invented~\cite{hinton2012improving}.}.\\ From a theoretical perspective the first two conditions are rather mild, in the sense that they only restrict the space of parametrizations and not the corresponding space of realizations. Specifically we can define the restricted parametrization space
\begin{align}\begin{split}
    \cN'_{(d,m,D)}:=\{\Gamma\in\cN_{(d,m,D)}\colon
    \|c^\Gamma_i\|_\infty=\|a^\Gamma_i\|_\infty\text{ for all } i\in\m\text{ and $\Gamma$ satisfies \ref{Cond2}}\}
\end{split}\end{align}
for which we have $\cR(\cN'_N)=\cR(\cN_N)$. Note that the above definition as well as the following definition and theorem are for networks with arbitrary output dimensions, as the balancedness condition makes this extension rather straightforward.\\
In order to satisfy Conditions~\ref{Cond3b} and~\ref{Cond3a} we need to restrict the parametrization space in a way which also restricts the corresponding space of realizations. One possibility to do so is the following approach, which also incorporates the previous restrictions as well as the transition to networks without biases. 
\begin{definition}[Restricted parametrization space] \label{def:restr_par}
Let $N=(d,m,D)\in\N^3$. We define
\begin{align}\label{def_Nstar}
    \cN_{N}^*:=\left\{\Gamma\in\cN'_{N}\colon (a^{\Gamma}_i)_{d-1}, (a^{\Gamma}_i)_d>0\ \text{ for all}\ i\in\m\right\}.
\end{align}
\end{definition}
While we no longer have $\cR(\cN^*_N)=\cR(\cN_N)$, Lemma~\ref{lem:Tech2} shows that for every
${\Theta\in\cP_{(d,m,D)}}$ there exists $\Gamma\in\cN^*_{(d+2,m+1,D)}$ such that for all $x\in\R^d$ it holds that 
\begin{align}
  \cR(\Gamma)(x_1,\dots,x_d,1,-1)=\cR(\Theta)(x_1,\dots,x_d).
\end{align}
In particular, this means that for any optimization problem over an unrestricted parametrization space $\cP_{(d,m,D)}$, there is a corresponding optimization problem over the parametrization space $\cN^*_{(d+2,m+1,D)}$ whose solution is at least as good (see Corollary~\ref{cor:inform}). 
Our main result now states that for such a restricted parametrization space we have uniform $(4,1/2)$ inverse stability w.r.t.\@ $|\cdot|_{\W}$, a proof of which can be found in Appendix~\ref{sec:AppProof3}.
\begin{theorem}[Inverse stability on $\cN_N^*$] \label{cor:main}
   Let $N\in\N^3$. For all $\Gamma\in\cN_N^*$ and $g\in\cR(\cN_N^*)$ there exists a parametrization $\Phi\in\cN_N^*$ with
   \begin{align} 
   \cR(\Phi)=g \quad \text{and} \quad \|\Phi-\Gamma\|_\infty\leq 4|g-\cR(\Gamma)|_{\W}^{\frac{1}{2}}.
   \end{align}
\end{theorem}

\section{Outlook} \label{sec:outlook}

This contribution investigates the potential insights which may be gained from studying the optimization problem over the space of realizations, as well as the difficulties encountered when trying to connect it to the parametrized problem. While Theorem~\ref{thm:optPara} and Theorem~\ref{cor:main} offer some compelling preliminary answers, there are multiple ways in which they can be extended.\\
To obtain our inverse stability result for shallow ReLU networks we studied sums of ridge functions. Extending this result to deep ReLU networks requires understanding their behaviour under composition. In particular, we have ridge functions which vanish on some half space, i.e.\@ colloquially speaking each neuron may \enquote{discard half the information} it receives from the previous layer.
This introduces a new type of degeneracy, which one will have to deal with.\\
Another interesting direction is an extension to inverse stability w.r.t.\@ some weaker norm like $\|\cdot\|_{L^\infty}$ or a fractional Sobolev norm under stronger restrictions on the space of parametrizations
 (see Lemma~\ref{lem:grad_est_tech} for a simple approach using very strong restrictions).\\
Lastly, note that Theorem~\ref{thm:optPara} is not specific to the ReLU activation function and thus also incentivizes the study of inverse stability for any other activation function.\\
From an applied point of view, Conditions~\ref{Cond1}-\ref{Cond3} motivate the implementation of corresponding regularization (i.e.\@ penalizing unbalancedness and redundancy in the sense of parallel weight vectors) in state-of-the-art networks, in order to explore whether preventing inverse stability leads to improved performance in practice. Note that there already are results using, e.g.\@ \emph{cosine similarity}, as regularizer to prevent parallel weight vectors~\cite{Bansal18,rodriguez2016regularizing} as well as approaches, called \emph{Sobolev Training}, reporting better generalization and data-efficiency by employing a Sobolev norm based loss~\cite{czarnecki2017sobolev}.
\section*{Acknowledgment}
The research of JB and DE was supported by the Austrian Science Fund (FWF) under
grants I3403-N32 and P 30148. The authors would like to thank Pavol Har\'{a}r for helpful comments. 
\bibliography{bib}
\newpage
\appendix
\section{Appendix - Proofs and Additional Material} 

\subsection{Section~\ref{sec:Sec1}} 
\subsubsection{Additional Material} \label{sec:AppAdd1}
\begin{example}[Without inverse stability: parameter minimum $\notimplies$ realization minimum] \label{ex:loc}
Consider the two domains
\begin{equation} 
    D_1:=\{ (x_1,x_2)\in (-1,1)^2 \colon x_2 > |x_1| \}, \quad 
    D_2 :=\{ (x_1,x_2) \in (-1,1)^2 \colon x_1 > |x_2| \}.
\end{equation}
For simplicity of presentation, assume we are given two samples $x^1\in D_1$, $x^2\in D_2$ with labels $y^1=0$, $y^2=1$. The corresponding MSE is
\begin{equation}
    \mathcal{L}(g)=\tfrac{1}{2}\big((g(x^1))^2+(g(x^2)-1)^2\big)
\end{equation} 
for every $g\in C(\R^2,\R)$.
Let the zero realization be parametrized by\footnote{See notation in the beginning of Section~\ref{sec:degen}.}
\begin{equation}
    \Gamma_*=(0,(-1,0))\in\cN_{(2,1,1)}
\end{equation}
with loss $\mathcal{L}(\cR(\Gamma_*))=\tfrac{1}{2}$.
Note that changing each weight by less than $\tfrac{1}{2}$ does not decrease the loss, as this rotates the vector $(-1,0)$ by at most $45^{\circ}$. Thus $\Gamma_*$ is a local minimum in the parametrization space. 
However, the sequence of realizations given by
\begin{equation}
    g_k(x)=\tfrac{1}{k} \ReLU (x_1-x_2)=\cR((1,-1),\tfrac{1}{k})
\end{equation}
satisfies that
\begin{equation}
    \| g_k - \cR(\Gamma_*)\|_{W^{1,\infty}((-1,1)^2)} = \| g_k \|_{W^{1,\infty}((-1,1)^2)} \le \tfrac{1}{k}  
\end{equation}
and
\begin{equation}
   \mathcal{L}(g_k)=\tfrac{1}{2}(g_k(x^2)-1)^2 < \tfrac{1}{2} = \mathcal{L}(\cR(\Gamma_*)),
\end{equation}
see Figure~\ref{fig:local}.
Accordingly, $\cR(\Gamma_*)$ is \emph{not} a local minimum in the realization space even w.r.t.\@ the Sobolev norm. The problem occurs, since inverse stability fails due to unbalancedness of $\Gamma_*$.
\begin{figure}[!htb]
    \centering
    \input{local_tikz.tex}
    \caption{The figure shows the samples $((x^i,y^i))_{i=1,2}$, the realization $\cR(\Gamma_*)$ of the local parameter minimum (left) and $g_3$ (right).}
    \label{fig:local}
\end{figure}
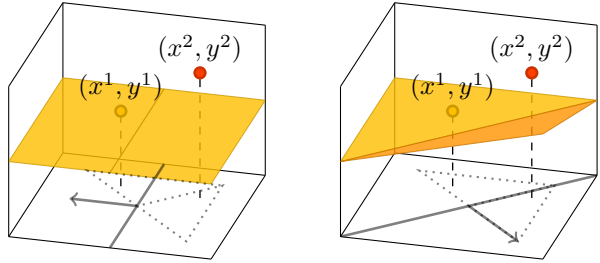
\end{example}

\begin{theorem}[Quality of local realization minima] \label{thm:opt}
Assume that
\begin{equation}\label{eq:approx}
    \supp_{f\in S}\ \infp_{\Phi\in\Omega_N}\|\cR(\Phi)-f\| < \eta \quad \text{(approximability)}.
\end{equation}
Let $g_*$ be a local minimum with radius $r'\geq 2\eta$ of the optimization problem $\min_{g\in\cR(\Omega_N)}\L(g)$. Then it holds for every $g\in\cR(\Omega_N)$ (in particular for every global minimizer) that
\begin{equation}\label{realization}
    \L(g_*)\leq \L(g) +\tfrac{2c}{r'}\|g_*-g\|\eta.
\end{equation}
\end{theorem}
\begin{proof}
Define $\lambda:=\frac{r'}{2\|g-g_*\|}$ and $f:=(1-\lambda)g_*+\lambda g\in S$. Due to \eqref{eq:approx} there is $\Phi\in\Omega_N$ such that $\|\cR(\Phi)-f\|\leq\eta$ and by the assumptions on $g_*$ and $\L$ it holds that
\begin{equation*}
    \L(g_*)\leq \L(\cR(\Phi))\leq \L(f) +c\eta \le (1-\lambda) \L(g_*) + \lambda  \L(g)+c\eta.
\end{equation*}
This completes the proof. See Figure~\ref{fig:opt} for illustration.
\begin{figure}[!htb]
    \centering
    \input{proof_tikz.tex}
    \caption{The figure illustrates the proof idea of Theorem~\ref{thm:opt}. Note that decreasing $\eta$, $c$, $\|g_*- g\|$ or increasing $r'$ leads to a better local minimum due to the convexity of the loss function (red).}
    \label{fig:opt}
\end{figure}
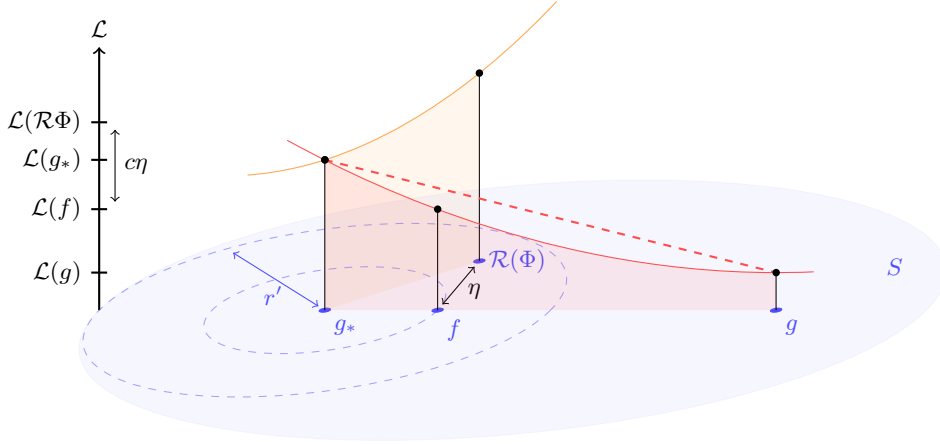
\end{proof}

\subsubsection{Proofs} \label{sec:AppProof1}
\begin{proof}[Proof of Proposition~\ref{prop:local}]
By Definition~\ref{def:InverseStability} we know that for every $g\in \cR(\Omega)$ with $\|g-\cR(\Gamma_*)\| \le (\frac{r}{s})^{1/\alpha} $ there exists $\Phi\in\Omega$ with
\begin{equation}
    \cR(\Phi)=g \quad \text{and} \quad \| \Phi -\Gamma_* \|_\infty \le s\|g- \cR(\Gamma_*)\|^\alpha\le r.
\end{equation}
Therefore by assumption it holds that
\begin{equation}
    \L(\cR(\Gamma_*))\leq\L(\cR(\Phi))=\L(g).
\end{equation}
which proves the claim.
\end{proof}

\begin{proof}[Proof of Theorem~\ref{thm:optPara}]
Let $\eps,r>0$, define $r':=(\tfrac{r}{s})^{1/\alpha}$ and $\eta:=\min\{(\tfrac{2c}{r'}\operatorname{diam}(S))^{-1}\eps,\tfrac{r'}{2}\}$. Then compactness of $S$ implies the existence of an architecture $n(\eps,r)\in\cA$ such that for every $N\in\cA$ with $N_1\ge n_1(\eps,r),\dots,N_{L-1}\ge n_{L-1}(\eps,r)$ the approximability assumption~\eqref{eq:approx} is satisfied. Let now $\Gamma_*$ be a local minimum with radius at least $r$ of $\min_{\Gamma\in\Omega_N}\L(\cR(\Gamma))$.
As we assume uniform $(s,\alpha)$ inverse stability,  Proposition~\ref{prop:local} implies that $\cR(\Gamma_*)$ is a local minimum of the optimization problem $\min_{g\in\cR(\Omega_N)}\L(g)$ with radius at least $r'=(\tfrac{r}{s})^{1/\alpha}\ge 2\eta$. Theorem~\ref{thm:opt} establishes the claim.
\end{proof}

\begin{proof}[Proof of Corollary~\ref{cor:inform}]
We simply combine the main observations from our paper. First, note that the assumptions imply that the restricted parametrization space $\Omega$, which we are optimizing over, is the space $\cN_{(d+2,N_1+1,D)}^*$ from Definition~\ref{def:restr_par}. Secondly, Theorem~\ref{cor:main} implies that the realization map is $(4,1/2)$ inverse stable on $\Omega$. Thus, Proposition~\ref{prop:local} directly proves Claim~\ref{it:col_first}. For the proof of Claim~\ref{it:col_sec} we make use of Lemma~\ref{lem:Tech2}.
It implies that for every $\Theta\in \cP_{(d,N_1,D)}$ there exists $\Gamma\in \Omega$ such that it holds that 
\begin{align}
   \tfrac{1}{n}\sum_{i=1}^n \|\cR(\Gamma)(\tilde{x}^i)-y^i\|^2 =  \tfrac{1}{n}\sum_{i=1}^n \|\cR(\Theta)(x^i)-y^i\|^2,
\end{align}
which proves the claim.
\end{proof}

\subsection{Section~\ref{sec:degen}} 
\subsubsection{Additional Material} \label{sec:AppAdd2}
\begin{lemma}[Reparametrization in case of linearly independent weight vectors]
\label{lem:triv}
Let  
\begin{equation}
    \Theta=(A^{\Theta},C^{\Theta})=\big([a_1^{\Theta} |\dots |a_m^{\Theta} ]^T, [c_1^{\Theta}| \dots | c_m^{\Theta} ]\big) \in\cN_{(d,m,D)}
\end{equation}
with linearly independent weight vectors $(a_i^{\Theta})_{i=1}^m$ and $\min_{i\in[m]} \|c_i^{\Theta}\|_\infty > 0$ and let 
\begin{equation}
    \Phi=(A^\Phi,B^\Phi)=\big([a_1^\Phi |\dots |a_m^\Phi ]^T, [c_1^\Phi| \dots | c_m^\Phi ]\big)\in\cN_{(d,m,D)}
\end{equation}
with $\cR(\Phi)=\cR(\Theta)$. Then there exists a permutation $\pi\colon [m]\to [m]$ such that for every $i\in[m]$ there exist $\lambda_i\in(0,\infty)$ with
\begin{equation}
    a_i^\Phi=\lambda_i a_{\pi(i)}^{\Theta} \quad \text{and} \quad c_i^\Phi =\tfrac{1}{\lambda_i}c_{\pi(i)}^{\Theta}.
\end{equation}
This means that, up to reordering and rebalancing, $\Theta$ is the unique parametrization of $\cR(\Theta)$.
\end{lemma}
\begin{proof}
First we define for every $s\in\{0,1\}^m$ the corresponding open orthant 
\begin{equation}
    O^s:=\{x\in\R^m\colon x_1(2s_1-1)> 0, \dots, x_m(2s_m-1)> 0 \} \subseteq\R^m.
\end{equation}
By assumption $A^\Theta$ has rank $m$, i.e.\@ is surjective, and therefore the preimages of the orthants 
\begin{equation}
    H^s:=\{  x\in\R^d \colon A^\Theta x\in O^s\}\subseteq \R^d, \quad s\in\{0,1 \}^m,
\end{equation}
are disjoint, non-empty open sets.
Note that on each $H^s$ the realization $\cR(\Theta)$ is linear with
\begin{equation}\label{eq:proof_der}
   \cR(\Theta)(x)=C^\Theta\operatorname{diag}(s)A^\Theta x 
   \quad \text{and} \quad D\cR(\Theta)(x)= C^\Theta\operatorname{diag}(s)A^\Theta.
\end{equation}
Since $A^\Theta$ has full row rank, it has a right inverse.
Thus we have for $s,t\in\{0,1\}^m$ that 
\begin{align}
   C^\Theta\operatorname{diag}(s)A^\Theta=C^\Theta\operatorname{diag}(t)A^\Theta
   \implies
   C^\Theta\operatorname{diag}(s)=C^\Theta\operatorname{diag}(t).
\end{align}
Note that $ C^\Theta\operatorname{diag}(s)=C^\Theta\operatorname{diag}(t)$ can only hold if $s=t$ due to the assumptions that $\|c^\Theta_i\|_\infty \neq 0$ for all $i\in[m]$.
Thus the above establishes that for $s,t\in\{0,1\}^m$ it holds that
\begin{equation}
    C^\Theta\operatorname{diag}(s)A^\Theta=C^\Theta\operatorname{diag}(t)A^\Theta \quad \text{if and only if} \quad  s=t,
\end{equation}
i.e.\@ $\cR(\Theta)$ has different derivatives on its $2^m$ linear regions.
In order for $\cR(\Phi)$ to have matching linear regions and matching derivatives on each one of them, there must exist a permutation matrix $P\in\{0,1\}^{m\times m}$ such that for every $s\in\{0,1\}^m$ 
\begin{equation}
    PA^\Phi x\in O^s \quad \text{for every } x\in H^s.
\end{equation}

Thus, there exist $(\lambda_i)_{i=1}^{m}\in (0,\infty)^m$ such that
\begin{equation}
    A^{\Phi}=\operatorname{diag}(\lambda_1,\dots,\lambda_m)P^TA^\Theta.
\end{equation}
The assumption that $D\cR(\Theta)=D\cR(\Psi)$, together with~\eqref{eq:proof_der} for $s=(1,\dots,1)$, implies that
\begin{equation}
    C^{\Phi}=C^\Theta P \operatorname{diag}(\tfrac{1}{\lambda_1}    ,\dots,\tfrac{1}{\lambda_m}),
\end{equation}
which proves the claim.
\end{proof}

\begin{example}[Failure due to unbalancedness]\label{Ex1b}
Let 
\begin{align}
    \Gamma_k:=\big((k,0),\tfrac{1}{k^2}\big)\in\cN_{(2,1,1)}, \quad k\in\N,
\end{align}
and $g_k\in\cR(\cN_{(2,1,1)})$ be given by
\begin{align}
    g_k(x)=\tfrac{1}{k}\rho(\<(0,1),x\>), \quad k\in\N.
\end{align}
The only way to parametrize $g_k$ is $g_k(x)=\cR(\Phi_{k})(x)=c\rho(\langle (0,a),x\rangle )$ with $a,c>0$ (see Lemma~\ref{lem:triv}), and we have
\begin{equation}
   |\cR(\Phi_k)-\cR(\Gamma_k)|_{\W}\leq\tfrac{1}{k}
   \quad \text{and} \quad  
   \|\Phi_k-\Gamma_k\|_\infty \geq k.
\end{equation}
\end{example}

\begin{lemma}
\label{lem:exact}
Let $d,m\in\N$ and $a_i\in\R^d$, $i\in\m$, such that $\sum_{i\in\m}a_i=0$. Then it holds for all $x\in\R^d$ that
\begin{align}
    \sum_{i\in\m}\rho(\<a_i,x\>)=\sum_{i\in\m}\rho(\<-a_i,x\>).
\end{align}
\end{lemma}
\begin{proof}
By assumption we have for all $x\in\R^d$ that $\sum_{i\in\m}\<a_i,x\>=0$. This implies for all $x\in\R^d$ that
\begin{align}
    \sum_{i\in\m\colon\<a_i,x\>\geq 0}\<a_i,x\>-\sum_{i\in\m}\<a_i,x\>=\sum_{i\in\m\colon\<a_i,x\>\leq 0}-\<a_i,x\>,
\end{align}
which proves the claim.
\end{proof}

\subsubsection{Proofs} \label{sec:AppProof2}

\begin{proof}[Proof of Example~\ref{Ex0}]
  We have for every $k\in\N$ that
  \begin{equation} \label{eq:diverge}
      \| g_k \|_{L^\infty((-1,1)^2)} \le \tfrac{1}{k} \quad \text{and} \quad |g_k|_{W^{1,\infty}} = k^2.
  \end{equation}
  Assume that there exists sequence of networks $(\Phi_k)_{k \in \N}\subseteq \cN_{(2,2,1)}$ with $\cR(\Phi_k)=g_k$
  and with uniformly bounded parameters, i.e.\@ $ \sup_{k\in\N} \|\Phi_k\|_\infty <\infty$. Note that there exists a constant $C$ (depending only on the network architecture) such that the realizations $\cR(\Phi_k)$ 
  are Lipschitz continuous with
  \begin{equation*}
      \operatorname{Lip}(\cR(\Phi_k)) \le C\|\Phi_k\|_\infty^2
  \end{equation*}
  (see~\cite[Prop. 5.1]{Petersen2018TopologicalSize}).
  It follows that $|\cR(\Phi_k)|_{W^{1,\infty}}\le \operatorname{Lip}(\cR(\Phi_k))$ is uniformly bounded which contradicts~\eqref{eq:diverge}.
\end{proof} 

\begin{proof}[Proof of Example~\ref{Ex1}]
The only way to parametrize $g_k$ is $g_k(x)=\cR(\Phi_{k})(x)=c\rho(\langle (0,a),x\rangle )$ with $a,c>0$ (see also Lemma~\ref{lem:triv}), which proves the claim. 
\end{proof}

\begin{proof}[Proof of Example~\ref{Ex2}]
Any parametrization of $g_k$ must be of the form $\Phi_k := (A,c)\in\R^{2 \times 2}\times\R^{1 \times 2}$ with
  \begin{equation}
      A= \begin{bmatrix} a_1 & 0 \\ 0 & a_2 \end{bmatrix} \quad \text{or} \quad 
      A= \begin{bmatrix} 0 & a_2 \\ a_1 & 0 \end{bmatrix}
  \end{equation}
  (see Lemma~\ref{lem:triv}).
  Thus it holds that $\|\Phi_k-\Gamma\|_\infty \geq \| (1,0) -(0,a_2)\|_\infty \geq 1$ and the proof is completed by direct calculation.
\end{proof}

\begin{proof}[Proof of Example~\ref{Ex3}]
Let $\Phi_k$ be an arbitrary parametrization of $g_k$ given by 
\begin{equation}
    \Phi_k = \left( [ \tilde{a}_1  | \tilde{a}_2  | \dots| \tilde{a}_{2m}]^T, \tilde{c}\right) \in \cN_{(d,2m,1)}
    \end{equation}
As $g_k$ has two linear regions separated by the hyperplane with normal vector $v$, there exists $j\in [2m]$ and $\lambda\in\R\setminus \{0\}$ such that 
\begin{equation}
    \tilde{a}_j= \lambda v.
\end{equation}
The distance of any weight vector $\pm a_i$ of $\Gamma$ to the line $\{\lambda v \colon \lambda \in \R\}$ can be lower bounded by
\begin{equation} \label{eq:dist}
 \| \pm a_i - \lambda v \|^2_\infty   \ge \tfrac{1}{d} \| \pm a_i - \lambda v \|_2^2 \ge \tfrac{1}{d^2} \left[ \|a_i\|^2_2\|v\|^2_2 - \langle a_i, v \rangle ^2\right] , \quad i\in[m], \lambda\in\R.
\end{equation}
The Cauchy-Schwarz inequality and the linear independence of $v$ to each $a_i$, $i\in[m]$, establishes that $C:=\tfrac{1}{d^2}\min_{i\in [m]}\left[ \|a_i\|^2_2\|v\|^2_2 - \langle a_i, v \rangle ^2\right] > 0$. 
Together with the fact that $\cR(\Gamma)=0$, this completes the proof.
\end{proof}

\begin{proof}[Proof of Example~\ref{Ex4}]
Since $x=\ReLU(x)-\ReLU(-x)$ for every $x\in\R$, the difference of the realizations is linear, i.e.
\begin{equation}
\cR(\Theta_k)-\cR(\Gamma_k)= \langle c_1^k a_1^k+c_2^k a_2^k+c_3^k a_3^k, x \rangle = \langle (0,0,3), x \rangle
\end{equation}
and thus the difference of the gradients is constant, i.e.
\begin{equation}
    |\cR(\Theta_k)-\cR(\Gamma_k)|_{\W}= 3, \quad k\in\N.
\end{equation}
However, regardless of the balancing and reordering of the weight vectors $a_i^k$, $i\in[3]$, we have that
\begin{equation}
   \|\Theta_k-\Gamma_k\|_\infty \geq  k.
\end{equation}
By Lemma~\ref{lem:triv}, up to balancing and reordering, there does not exist any other parametrization of $\Theta_k$ with the same realization. 
\end{proof}

\subsection{Section~\ref{sec:main}}
\subsubsection{Additional Material} \label{sec:AppAdd3}

\begin{lemma}\label{lem:Tech2}
Let $d,m,D\in\N$ and $\Theta\in\cP_{(d,m,D)}$. Then there exists $\Gamma\in\cN^*_{(d+2,m+1,D)}$ such that for all $x\in\R^d$ it holds that 
\begin{align}
  \cR(\Gamma)(x_1,\dots,x_d,1,-1)=\cR(\Theta)(x).
\end{align}
\end{lemma}
\begin{proof}
Since $\Theta\in\cP_{(d,m,D)}$ it can be written as
\begin{align}
    \Theta= \Big( \big( A,b\big), \big(c,e\big) \Big)= \left(\big( [ a_1 | \dots | a_{m} ]^T , b \big),([ c_1 | \dots | c_{m} ],e) \right) 
\end{align}
with 
\begin{align}
    \cR(\Theta)(x)=\sum^m_{i=1} c_i \rho(\< a_i,x\>+b_i)+e, \quad x\in\R^d,
\end{align}
where $A\in\R^{m\times d}$, $b\in\R^m$, $C\in\R^{D\times m}$, and $e\in\R^D$. 
We define for $i\in\m$ 
\begin{align}
    b^+_i:=\begin{cases}b_i+1 & \colon b_i\geq 0\\ 1 & \colon b_i < 0 \end{cases},\quad\text{and}\quad
    b^-_i:=\begin{cases}1 & \colon b_i\geq 0\\ -b_i+1 & \colon b_i < 0 \end{cases}
\end{align}
and observe that $b_i^+>0$, $b_i^->0$, and $b_i^+ - b_i^-=b_i$.
For $i\in\oneto{m}$ let
\begin{align}
    c^*_i:=\begin{cases}c_i & \colon \|c_i\|_\infty\neq 0 \\ (1,\dots,1) & \colon \|c_i\|_\infty=0\end{cases}
\end{align}
and
\begin{align}
    a^*_i:=\begin{cases}(a_{i,1},\dots,a_{i,d},b^+_i,b^-_i) & \colon \|c_i\|_\infty\neq 0 \\ (0,\dots,0,1,1) & \colon \|c_i\|_\infty=0\end{cases}.
\end{align}
Note that we have
\begin{align}
    \cR(\Theta)(x)=\sum^m_{i=1} c^*_i \rho(\< a_i^*,(x_1,\dots,x_d,1,-1)\>)+e, \quad x\in\R^d.
\end{align}
To include the second bias $e$ 
let
\begin{align}
    c^*_{m+1}:=\begin{cases} e & \colon e\neq 0 \\ (1,\dots,1) & \colon e= 0\end{cases},\quad\text{and}\quad a^*_{m+1}:=\begin{cases}(0,\dots,0,2,1) & \colon e\neq 0 \\ (0,\dots,0,1,1) & \colon e = 0\end{cases}.
\end{align}
In order to balance the network, let $a_i^\Gamma=a_i^*(\tfrac{\|c^*_i\|_\infty}{\|a_i^*\|_\infty})^{1/2}$ and $c_i^\Gamma=c_i^*(\tfrac{\|a^*_i\|_\infty}{\|c_i^*\|_\infty})^{1/2}$ for every $i\in[m+1]$. Then the claim follows by direct computation.
\end{proof}

\subsubsection{Proofs} \label{sec:AppProof3}

\begin{proof}[Proof of Theorem~\ref{thm:main}]
    Without loss of generality\footnote{In case one of them is zero, the other one can be set to zero without changing the realization.}, we can assume for all $i\in\m$ that $a^\Theta_i=0$ if and only if $c^\Theta_i=0$. We now need to show that there always exists a way to reparametrize $\cR(\Theta)$ such that the architecture remains the same and \eqref{main_result} is satisfied. For simplicity of notation we will write $r:=|g-\cR(\Gamma)|_{\W}$ throughout the proof. Let $f_i^\Gamma\colon\R^d\to\R$ resp. $f_i^\Theta\colon\R^d\to\R$ be the part that is contributed by the $i$-th neuron, i.e.
   \begin{align}
     &\cR(\Gamma)=\sum_{i=1}^m f^\Gamma_i \quad\text{with}\quad f^\Gamma_i(x):=c^\Gamma_i\rho(\<a^\Gamma_i,x\>),\\
     g=&\cR(\Theta)=\sum_{i=1}^m f^\Theta_i
     \quad\text{with}\quad
     f^\Theta_i(x):=c^\Theta_i\rho(\<a^\Theta_i,x\>).
   \end{align}
   Further let
   \begin{align}\begin{split}
   H_{\Gamma,i}^+&:=\{x\in\R^d\colon \<a^\Gamma_i,x\> > 0\},\\  H_{\Gamma,i}^0&:=\{x\in\R^d\colon \<a^\Gamma_i,x\> = 0\},\\
   H_{\Gamma,i}^-&:=\{x\in\R^d\colon \<a^\Gamma_i,x\> < 0\}.
   \end{split}\end{align}
    By conditions \ref{Cond2} and \ref{Cond3b} we have for all $i,j\in
    I^\Gamma$ that
    \begin{align}\label{eq:hyperplane_assump1}
      i\neq j \implies H^0_{\Gamma,i}\neq H^0_{\Gamma,j}.
   \end{align}
   Further note that we can reparametrize $\cR(\Theta)$ such that the same holds there.
   To this end observe that 
   \begin{align} \label{eq:redundant}
   c\rho(\<a,x\>)+c'\rho(\<a',x\>)=(c+c'\tfrac{\|a'\|_\infty}{\|a\|_\infty})\rho(\<a,x\>),
   \end{align}
   given that $a'$ is a positive multiple of $a$. Specifically, let $(J_k)_{k=1}^K$ be a partition of $I^\Theta$ (i.e.\@ $J_k\neq\emptyset$, $\cup_{k=1}^K J_k=I^\Theta$ and $J_k\cap J_{k'}=\emptyset$ if $k\neq k'$), such that for all $k\in\oneto{K}$ it holds that
   \begin{align}
       i,j\in J_k \implies  \frac{ a_j^\Theta}{\|a^\Theta_j\|_\infty}=\frac{ a^\Theta_i}{\|a^\Theta_i\|_\infty}.
   \end{align}
   We denote by $j_k$ the smallest element in $J_k$ and make the following replacements, for all $i\in I^\Theta$, without changing the realization of $\Theta$:
   \begin{alignat}{2} 
       a^\Theta_i&\mapsto a^\Theta_i, c^\Theta_i\mapsto\sum_{j\in J_k}c^\Theta_j\tfrac{\|a^\Theta_j\|_\infty}{\|a^\Theta_{j_k}\|_\infty}  ,\quad\ &\text{if}\ i\in J_k\ \text{and}\ i=j_k,\label{eq:replace1}\\
       a^\Theta_i&\mapsto 0, c^\Theta_i\mapsto 0,\quad\ &\text{if}\ i\in J_k\ \text{and}\ i\neq j_k. \label{eq:replace2}
   \end{alignat}
   Note that we also update the set $I^\Theta:=\{i\in\m\colon a^\Theta_i\neq 0\}$ accordingly.
   Let now
  \begin{align}\begin{split}
   H_{\Theta,i}^+&:=\{x\in\R^d\colon \<a^\Theta_i,x\> > 0\},\\ H_{\Theta,i}^0&:=\{x\in\R^d\colon \<a^\Theta_i,x\> = 0\},\\
   H_{\Theta,i}^-&:=\{x\in\R^d\colon \<a^\Theta_i,x\> > 0\}.
   \end{split}\end{align}
  By construction and condition~\ref{Cond3b}, we have for all $i,j\in I^\Theta$ that      
    \begin{align}\label{eq:hyperplane_assump2}
      i\neq j \implies H^0_{\Theta,i}\neq H^0_{\Theta,j}.
   \end{align}
   Note that we now have a parametrization $\Theta$ of $g$, where all weight vectors $a_i^\Theta$ are either zero (in which case the corresponding $c_i^\Theta$ are also zero) or pairwise linearly independent to each other nonzero weight vector.\\  
   Next, for $s\in\{0,1\}^{m}$, let
   \begin{align}\begin{split}
    H^s_\Gamma &:=\bigcap_{i\in\m\colon s_i=1}H^+_{\Gamma,i}\cap\bigcap_{i\in\m\colon s_i=0}H^-_{\Gamma,i},\\
    H^s_\Theta &:=\bigcap_{i\in\m\colon s_i=1}H^+_{\Theta,i}\cap\bigcap_{i\in\m\colon s_i=0}H^-_{\Theta,i},
   \end{split}\end{align}
   and  
   	\begin{align}\begin{split}
   		S^\Gamma :=\{s\in\{0,1\}^{m}\colon H^s_\Gamma\neq\emptyset\},\quad 
   		S^\Theta :=\{s\in\{0,1\}^{m}\colon H^s_\Theta\neq\emptyset\}.
	\end{split}\end{align}
	The $H^s_\Gamma$, $s\in S^\Gamma$, and $H^s_\Theta$, $s\in S^\Theta$, are the interiors of the different linear regions of $\cR(\Gamma)$ and $\cR(\Theta)$ respectively.
	 Next observe that the derivatives of $f^\Gamma_i,f^\Theta_i$ are (a.e.) given by 
   \begin{align}\begin{split}
   Df^\Gamma_i(x)=\1_{H^+_{\Gamma,i}}\!\!(x)\,c^\Gamma_i a^\Gamma_i,\quad 
   Df^\Theta_i(x)=\1_{H^+_{\Theta,i}}\!\!(x)\,c^\Theta_i a^\Theta_i. 
   \end{split}\end{align}
   Note that for every $x\in H^s_\Gamma$, $y\in H^s_\Theta$ we have 
   \begin{align}\begin{split}
   	 D\cR(\Gamma)(x)&=\sum_{i\in\m}Df^\Gamma_i(x)
   	 =\sum_{i\in\m}s_i c^\Gamma_i a^\Gamma_i=:\Sigma^\Gamma_s,\\
   	 D\cR(\Theta)(y)&=\sum_{i\in\m}Df^\Theta_i(y)=\sum_{i\in\m}s_i c^\Theta_i a^\Theta_i=:\Sigma^\Theta_s.
   \end{split}\end{align}
   Next we use that for $s\in S^\Gamma$, $t\in S^\Theta$ we have $|\Sigma^\Gamma_s-\Sigma^\Theta_t|\leq r$ if $H^\Gamma_s\cap H^\Theta_t\neq\emptyset$, and compare adjacent linear regions of $\cR(\Gamma)-\cR(\Theta)$.
   Let now $i\in I^\Gamma$ and consider the following cases:\\
   \textbf{Case 1}: We have $H^0_{\Gamma,i}\neq H^0_{\Theta,j}$ for all $j\in I^\Theta$. This means that the $Df_k^\Theta$, $k\in\m$, and the $Df_k^\Gamma$, $k\in\m\backslash\{i\}$, are the same on both sides near the hyperplane $H^0_{\Gamma,i}$, while the value of $Df^\Gamma_i$ is $0$ on one side and $c_i^\Gamma a_i^\Gamma$ on the other. Specifically, there exist $s^+,s^-\in S^\Gamma$ and $s^*\in S^\Theta$ such that $s^+_i=1$, $s^-_i=0$, $s^+_j=s^-_j$ for all $j\in\m\backslash\{i\}$, and $H_\Gamma^{s^+}\cap H_\Theta^{s^*}\neq\emptyset$, $H_\Gamma^{s^-}\cap H_\Theta^{s^*}\neq\emptyset$, which implies
   \begin{align}\label{Case1Est}
   	\| c^\Gamma_i a^\Gamma_i\|_\infty=\|(\Sigma^\Gamma_{s^+}-\Sigma^\Theta_{s^*})-(\Sigma^				\Gamma_{s^-}-\Sigma^\Theta_{s^*})\|_\infty \leq 2r.
	\end{align} 
	\textbf{Case 2}: There exists $j\in I^\Theta$ such that $H^0_{\Gamma,i}=H^0_{\Theta,j}$. Note that \eqref{eq:hyperplane_assump1} ensures that $H^0_{\Gamma,i}\neq H^0_{\Gamma,k}$ for $k\in\oneto{m}\setminus\{i\}$ and \eqref{eq:hyperplane_assump2} ensures that $H^0_{\Theta,j}\neq H^0_{\Gamma,k}$ for $k\in\oneto{m}\setminus\{j\}$. Moreover, Condition~\ref{Cond3a} implies $H^+_{\Gamma,i}=H^+_{\Theta,j}$.
	This means that the $Df_k^\Theta$, $k\in\m\backslash\{j\}$, and the $Df_k^\Gamma$, $k\in\m\backslash\{i\}$, are the same on both sides near the hyperplane $H^0_{\Gamma,i}=H^0_{\Theta,j}$, while the values of $Df^\Gamma_i$ and $Df^\Theta_j$ change. Specifically there exist $s^+,s^-\in S^\Gamma$ and $t^+,t^-\in S^\Theta$ such that $s^+_i=1$, $s^-_i=0$, $s^+_k=s^-_k$ for all $k\in\m\backslash\{i\}$, $t^+_j=1$, $t^-_j=0$, $t^+_k=t^-_k$ for all $k\in\m\backslash\{j\}$ and $H^\Gamma_{s^+}\cap H^\Theta_{t^+}\neq\emptyset$, $H^\Gamma_{s^-}\cap H^\Theta_{t^-}\neq\emptyset$, which implies
      \begin{align}\label{Case2Est}
   	\| c^\Gamma_i a^\Gamma_i-c^\Theta_j a^\Theta_j\|_\infty=\|(\Sigma^\Gamma_{s^+}-\Sigma^\Theta_{t^+})-(\Sigma^				\Gamma_{s^-}-\Sigma^\Theta_{t^-})\|_\infty\leq 2r.
	\end{align}
	Analogously we get for $i\in I^\Theta$ that $H^0_{\Theta,i}\neq H^0_{\Gamma,j}$ for all $j\in I^\Gamma$ implies $\| c^\Theta_i a^\Theta_i\|_\infty\leq 2r$.
Next let
 \begin{align}\label{Ione}
     I_1:=\{i\in\m\colon H^0_{\Gamma,i}\neq H^0_{\Theta,j}\ \text{for all}\ j\in I^\Theta\}\cup\{i\in\m\colon a^\Gamma_i=0\}
 \end{align}
 and
 \begin{align}\label{Itwo}
     I_2:= [m]\setminus I_1=\{i\in\m\colon \exists\ j\in I^\Theta\ \text{such that}\ H^+_{\Gamma,i}=H^+_{\Theta,j}\}.
 \end{align}
 Colloquially speaking, this shows that for every $f^\Gamma_i$ with $i\in I_2$ there is a $f^\Theta_j$ with exactly matching half-spaces, i.e. $H^+_{\Gamma,i}=H^+_{\Theta,j}$, and approximately matching gradients (Case 2). Moreover, all unmatched $f^\Gamma_i$ and $f^\Theta_j$ must have a small gradient (Case 1).\\
Specifically, the above establishes that there exists a permutation $\pi\colon\oneto{m}\to\oneto{m}$ such that for every $i\in I_1$ it holds that
   \begin{align}\label{small_cont}
       \| c^\Gamma_i a^\Gamma_i\|_\infty,\|c^\Theta_{\pi(i)} a^\Theta_{\pi(i)}\|_\infty\leq 2r,
   \end{align}
   and for every $i\in I_2$ that
    \begin{align}\label{eq:difference}
   	\| c^\Gamma_{i}a^\Gamma_{i}-c^\Theta_{\pi(i)} a^\Theta_{\pi(i)}\|_\infty\leq 2r.
	\end{align}
   We make the following replacements, for all $i\in [m]$, without changing the realization of $\Theta$:
   \begin{align}\label{repara1}
       a^\Theta_i\to a^\Theta_{\pi(i)},\quad
       c^\Theta_i\to c^\Theta_{\pi(i)}.
   \end{align}
   In order to balance the weights of $\Theta$ for $I_1$, we further make the following replacements, for all $i\in I_1$ with $a^\Theta_i\neq0$, without changing the realization of $\Theta$: 
      \begin{align}\label{repara2}
       a^\Theta_i\to (\tfrac{|c_i^\Theta|}{\|a_i^\Theta\|_\infty})^{1/2}\ a^\Theta_i,\quad
       c^\Theta_i\to (\tfrac{\|a_i^\Theta\|_\infty}{|c_i^\Theta|})^{1/2}\ c^\Theta_i.
   \end{align}
   This implies for every $i\in I_1$ that
   \begin{align}\label{eq61}
        |c^\Theta_{i}|,  \|a^\Theta_{i}\|_\infty\leq (2r)^{1/2}.
   \end{align}
   Moreover, due to Condition~\ref{Cond1}, we get for every $i\in I_1$ that
   \begin{align}
               |c^\Gamma_{i}|,  \|a^\Gamma_{i}\|_\infty\leq \beta.
   \end{align}
   Thus we get for every $i\in I_1$ that
    \begin{align}\label{estI1}
        |c^\Theta_i- c^\Gamma_i|,  \|a^\Theta_i-a^\Gamma_i\|_\infty\leq \beta+(2r)^{1/2}.
    \end{align}
    Next we (approximately) match the balancing of $(c_i^\Theta,a_i^\Theta)$ to the balancing of $(c_i^\Gamma,a_i^\Gamma)$ for $i\in I_2$, in order to derive estimates on $ |c^\Theta_i- c^\Gamma_i|$ and $\|a^\Theta_i-a^\Gamma_i\|_\infty$ from \eqref{eq:difference}. 
    Specifically, we make the following replacements, for all $i\in I_2$, without changing the realization of $\Theta$:
      \begin{alignat}{2}
       a^\Theta_i&\to (\tfrac{|c_i^\Theta|}{\|a_i^\Theta\|_\infty})^{1/2}\ a^\Theta_i,\quad c^\Theta_i\to (\tfrac{\|a_i^\Theta\|_\infty}{|c_i^\Theta|})^{1/2}\ c^\Theta_i, &\text{if}\ \|c_i^\Gamma a_i^\Gamma\|_\infty\leq 2r,\label{eq:ReparaA}\\
       a^\Theta_i&\to \frac{c_i^\Theta }{c_i^\Gamma}a^\Theta_i,\quad c^\Theta_i\to c^\Gamma_i, &\text{if}\ \|c_i^\Gamma a_i^\Gamma\|_\infty> 2r,|c^\Gamma_i|>\|a^\Gamma_i\|_\infty,\label{eq:ReparaB1}\\ 
       a^\Theta_i&\to a^\Gamma_i,\quad c^\Theta_i\to \frac{\|a^\Theta_i\|_\infty}{\|a^\Gamma_i\|_\infty}c^\Theta_i,&\text{if}\ \|c_i^\Gamma a_i^\Gamma\|_\infty> 2r,|c^\Gamma_i|<\|a^\Gamma_i\|_\infty\label{eq:ReparaB2},\\
       a^\Theta_i&\to (\tfrac{|c_i^\Theta|}{\|a_i^\Theta\|_\infty})^{1/2}\ a^\Theta_i,\quad
       c^\Theta_i\to (\tfrac{\|a_i^\Theta\|_\infty}{|c_i^\Theta|})^{1/2}\ c^\Theta_i, \quad&\text{if}\ \|c_i^\Gamma a_i^\Gamma\|_\infty> 2r,|c^\Gamma_i|=\|a^\Gamma_i\|_\infty.\label{eq:ReparaB3}
   \end{alignat}
    Let now $i\in I_2$ and consider the following cases:\\
    \textbf{Case A}: We have $\|c_i^\Gamma a_i^\Gamma\|_\infty\leq 2r$ which, together with~\eqref{eq:difference}, implies $\|c_i^\Theta a^\Theta_i\|_\infty \leq 4r$. Due to \eqref{eq:ReparaA} and Condition~\ref{Cond1} it follows that
    \begin{align}\label{eq:caseA}
         |c^\Theta_i- c^\Gamma_i|,  \|a^\Theta_i-a^\Gamma_i\|_\infty\leq \beta + 2r^{1/2}.
    \end{align}
    \textbf{Case B.1}: We have $\|c_i^\Gamma a_i^\Gamma\|_\infty>2r$ and $|c^\Gamma_i|>\|a_i^\Gamma\|_\infty$ which ensures $|c_i^\Gamma|>\|c_i^\Gamma a_i^\Gamma\|_\infty^{1/2}$. Due to \eqref{eq:ReparaB1} we get $c_i^\Theta=c_i^\Gamma$ and it follows that
    \begin{align}\label{eq:caseB1}
        \|a^\Theta_i-a^\Gamma_i\|_\infty=\frac{1}{|c_i^\Gamma|}\|c^\Theta_i a^\Theta_i-c^\Gamma_i a^\Gamma_i\|_\infty\leq\frac{2r}{\|c_i^\Gamma a_i^\Gamma\|_\infty^{1/2}}\leq (2r)^{1/2}.
    \end{align}
    \textbf{Case B.2}: We have $\|c_i^\Gamma a_i^\Gamma\|_\infty>2r$ and $|c^\Gamma_i| <\|a_i^\Gamma\|_\infty$ which ensures $\|a_i^\Gamma\|>\|c_i^\Gamma a_i^\Gamma\|_\infty^{1/2}$. Due to \eqref{eq:ReparaB2} we get $a_i^\Theta=a_i^\Gamma$ and it follows that
    \begin{align}\label{eq:caseB2}
        |c^\Theta_i-c^\Gamma_i|=\frac{1}{\|a^\Gamma_i\|_\infty}\|c^\Theta_i a^\Theta_i-c^\Gamma_i a^\Gamma_i\|_\infty\leq\frac{2r}{\|c_i^\Gamma a_i^\Gamma\|_\infty^{1/2}}\leq (2r)^{1/2}.
    \end{align}
    \textbf{Case B.3}: We have $\|c_i^\Gamma a_i^\Gamma\|_\infty>2r$ and $|c^\Gamma_i|=\|a_i^\Gamma\|_\infty$. Note that $\|c_i^\Gamma a_i^\Gamma\|_\infty>2r$ and \eqref{eq:difference} ensure that $\sgn(c_i^\Theta)=\sgn(c_i^\Gamma)$, and that for $x,y>0$ it holds that $|x-y|\leq|x^2-y^2|^{1/2}$. Combining this with the definition of $I_2$, the reverse triangle inequality, and~\eqref{eq:ReparaB3} implies that
    \begin{align}\label{eq:caseB3}
       \|a^\Theta_i-a^\Gamma_i\|_\infty\leq(2r)^{1/2} \quad \text{and}\quad |c^\Theta_i-c^\Gamma_i\big|\leq(2r)^{1/2}.
    \end{align} 
        Combining \eqref{estI1}, \eqref{eq:caseA}, \eqref{eq:caseB1}, \eqref{eq:caseB2}, and \eqref{eq:caseB3} establishes that
    \begin{align}
        \|\Theta-\Gamma\|_\infty\leq\beta+2r^{\frac{1}{2}},
    \end{align}
    which completes the proof.
\end{proof}

\begin{proof}[Proof of Theorem~\ref{cor:main}]
Let $\Theta\in\cN_N^*$ be a parametrization of $g$, i.e.\@ $\cR(\Theta)=g$. We write 
\begin{align}
    \Gamma=
    \Big(\left[\begin{array}{c} a_1^\Gamma  \\ \hline \vdots \\ \hline \vphantom{\Big(}a_m^\Gamma \end{array}\right], 
    \big[ c_1^\Gamma \big| \dots \big| c_m^\Gamma \big] \Big), \quad 
    \Theta=
    \Big(\left[\begin{array}{c} a_1^\Theta \\  \hline \vdots \\ \hline  \vphantom{\Big(}a_m^\Theta \end{array}\right],
    \big[ c_1^\Theta \big| \dots \big| c_m^\Theta \big] \Big)\in\cN_{(d,m,D)}^*
\end{align}
and $r:=|g-\cR(\Gamma)|_{\W}$. For convenience of notation we consider the weight vectors $a_i^\Gamma$, $a_i^\Theta$ here as row vectors in order to write the derivatives of the ridge functions as $c^\Gamma_i a^\Gamma_i,\ c^\Theta_i a^\Theta_i \in\R^{D \times d}$ without transposing.\\
We will now adjust the approach used in the proof of Theorem~\ref{thm:main} to work for multi-dimensional outputs in the case of balanced networks. 
By definition of $\cN_N^*$, the $(a_i^\Theta)_{i=1}^m$ are pairwise linearly independent and we can skip the first reparametrization step in~\eqref{eq:replace1} and~\eqref{eq:replace2}.\\
The following \enquote{hyperplane-jumping} argument, which was used to get the estimates \eqref{Case1Est} and \eqref{Case2Est}, works analogously since Conditions \ref{Cond2} and \ref{Cond3} are fulfilled by definition of $\cN_N^*$. This establishes the existence of a permutation $\pi\colon\oneto{m}\to\oneto{m}$ and sets $I_1,I_2\subseteq [m]$, as defined as in~\eqref{Ione} and~\eqref{Itwo}, such that for every $i\in I_1$ it holds that
   \begin{align} \label{eq:reorder_mult}
       \| c^\Gamma_i a^\Gamma_i\|_\infty,\|c^\Theta_{\pi(i)} a^\Theta_{\pi(i)}\|_\infty\leq 2r,
   \end{align}
and for every $i\in I_2$ that
\begin{align}\label{eq:difference_mult}
   \| c^\Gamma_{i}a^\Gamma_{i}-c^\Theta_{\pi(i)} a^\Theta_{\pi(i)}\|_\infty\leq 2r.
\end{align}
As in~\eqref{repara1}, we make the following replacements, for all $i\in [m]$, without changing the realization of $\Theta$:
\begin{align}\label{repara1_mult}
   a^\Theta_i\to a^\Theta_{\pi(i)},\quad
   c^\Theta_i\to c^\Theta_{\pi(i)}.
\end{align}
Note that the weights of $\Theta$ are already balanced, i.e. we have for every $i\in\oneto{m}$ that
\begin{equation}
    \| c^{\Theta}_i \|_\infty=\| a^{\Theta}_i \|_\infty = \|c^\Theta_i \|_\infty^{1/2} \| a^\Theta_i\|_\infty^{1/2} = \|c^\Theta_i a^\Theta_i\|_\infty^{1/2}.
\end{equation}
Thus, we can skip the reparametrization step in~\eqref{repara2} and get directly for every $i\in I_1$ that
\begin{align}
    \|c^\Theta_i- c^\Gamma_i\|_\infty \leq \|c^\Theta_i\|_\infty + \|c^\Gamma_i\|_\infty = \|c^\Theta_i a^\Theta_i\|_\infty^{1/2} + \|c^\Gamma_i a^\Gamma_i\|_\infty^{1/2} \leq 2(2r)^{1/2}
\end{align}
and analogously $\|a^\Theta_i-a^\Gamma_i\|_\infty\leq 2(2r)^{1/2}$.\\ 
For $i\in I_2$ we need to slightly deviate from the proof of Theorem~\ref{thm:main}.
We can skip the reparametrization step in \eqref{eq:ReparaA}-\eqref{eq:ReparaB3} due to balancedness and need to distinguish three cases: \\
\textbf{Case A.1}: We have $\|c_i^\Gamma a_i^\Gamma\|_\infty\leq 2r$ which, together with~\eqref{eq:difference_mult}, implies $\|c_i^\Theta a^\Theta_i\|_\infty \leq 4r$. Due to balancedness it follows that
\begin{align}\label{eq:caseA1}
     \|c^\Theta_i- c^\Gamma_i\|_\infty,  \|a^\Theta_i-a^\Gamma_i\|_\infty\leq 4r^{1/2}.
\end{align}
\textbf{Case A.2}: We have $\|c_i^\Theta a_i^\Theta\|_\infty\leq 2r$ which, together with~\eqref{eq:difference_mult}, implies $\|c_i^\Gamma a^\Gamma_i\|_\infty \leq 4r$. Again it follows that
\begin{align}\label{eq:caseA2}
     \|c^\Theta_i- c^\Gamma_i\|_\infty,  \|a^\Theta_i-a^\Gamma_i\|_\infty\leq 4r^{1/2}.
\end{align}
    \textbf{Case B}: We have $\|c_i^\Theta a_i^\Theta\|_\infty> 2r$ and $\|c_i^\Gamma a_i^\Gamma\|_\infty > 2r$. Due to the definition of $I_2$ there exists $e_i\in\R^d$, $\lambda_i^\Gamma,\lambda_i^\Theta\in (0,\infty)$ with $\|e_i\|_\infty=1$, $a_i^\Theta=\lambda_i^\Theta e_i$, and $a_i^\Gamma=\lambda_i^\Gamma e_i$. As in~\eqref{eq:caseB3} we obtain that
    \begin{equation} \label{eq:caseB}
    \begin{split}
        \|a^\Theta_i-a^\Gamma_i\|_\infty&=\|e_i\|_\infty |\lambda^\Theta_i-\lambda^\Gamma_i| \le |(\lambda^\Theta_i)^2-(\lambda^\Gamma_i)^2|^{1/2} 
        \\ &=| \|c^\Theta_i \|_\infty\| a^\Theta_i\|_\infty -  \|c^\Gamma_i \|_\infty\| a^\Gamma_i\|_\infty |^{1/2} 
        \\ &\leq \|c^\Theta_i  a^\Theta_i -  c^\Gamma_i  a^\Gamma_i\|_\infty^{1/2} \leq(2r)^{1/2}.
    \end{split}
    \end{equation}
    Let now w.l.o.g.\@ $\|a^\Gamma_i\|_\infty \ge \|a^\Theta_i\|_\infty$ (otherwise we switch their roles in the following) which implies that $\lambda^\Gamma_i = \Delta_i+\lambda^\Theta_i$ with $\Delta_i=\lambda^\Gamma_i - \lambda^\Theta_i\geq0$. Then it holds that
    \begin{equation}
    \begin{split}
        \|c^\Theta_i-c^\Gamma_i\|_\infty &= \frac{\|c^\Theta_i a^\Gamma_i-c^\Gamma_i a^\Gamma_i\|_\infty}{\|a^\Gamma_i\|_\infty}\le  \frac{\|c^\Theta_i a^\Gamma_i-c^\Theta_i a^\Theta_i \|_\infty+\|c^\Theta_i a^\Theta_i-c^\Gamma_i a^\Gamma_i\|_\infty}{\|a^\Gamma_i\|_\infty}
        \\&\le \frac{\|c^\Theta_i\|_\infty |\lambda^\Gamma_i-\lambda^\Theta_i |+2r}{\lambda^\Gamma_i}
        =  \frac{\lambda^\Theta_i \Delta_i+2r}{\Delta_i+\lambda^\Theta_i}
        \\&=  \frac{(2r)^{1/2}(\Delta_i+\lambda^\Theta_i) -(\lambda^\Theta_i-(2r)^{1/2})((2r)^{1/2}-\Delta_i)}{\Delta_i+\lambda^\Theta_i}\le (2r)^{1/2}.
    \end{split}
    \end{equation}
    The last step holds due to~\eqref{eq:caseB} and the balancedness of $\Theta$ which ensure that
    \begin{align}
        \lambda^\Theta_i =\|c^\Theta_i a^\Theta_i\|_\infty^{1/2} > (2r)^{1/2} \ge | \lambda^\Theta_i-\lambda^\Gamma_i|=\Delta_i.
    \end{align}
    This completes the proof.
\end{proof}

\subsection{Section~\ref{sec:outlook}}
\subsubsection{Additional Material} \label{sec:AppAdd4}

\begin{lemma}[Inverse stability for fixed weight vectors]\label{lem:grad_est_tech}
Let $N=(d,m,D)\in\N^3$, let $A=[a_1|\dots |a_m]^T \in \R^{m\times d}$ with 
\begin{align}
 \frac{a_i}{\|a_i\|_\infty}\neq\frac{a_j}{\|a_j\|_\infty} \quad \text{and} \quad 
&(a_i)_{d-1}, (a_i)_d>0
\end{align}
for all $i\in\m, j\in\oneto{m}\setminus\{i\}$, and define
\begin{align}\begin{split}\label{def_NA}
    \cN_{N}^A:=\left\{\Gamma\in\cN_N \colon 
    a^\Gamma_i=\lambda_i a_i \text{ with } \lambda_i\in(0,\infty)  \text{ and } \|c^\Gamma_i\|_\infty= \|a^\Gamma_i\|_\infty \ \  \text{for all}\ i\in\m  \right\}.
\end{split}\end{align}
Then for every $B\in (0,\infty)$ there is $C_B\in(0,\infty)$ such that we have uniform $(C_B,1/2)$ inverse stability w.r.t.\@ $\|\cdot\|_{L^\infty((-B,B)^d)}$. That is, for all $\Gamma\in\cN_N^A$ and $g\in\cR(\cN_{N}^A)$ there exists a parametrization $\Phi\in\cN_{N}^A$ with
   \begin{align} 
   \cR(\Phi)=g \quad \text{and} \quad \|\Phi-\Gamma\|_\infty\leq C_B\|g-\cR(\Gamma)\|_{L^{\infty}((-B,B)^d)}^{\frac{1}{2}}.
   \end{align} 

\end{lemma}
\begin{proof}
Note that the non-zero angle between the hyperplanes given by the weight vectors $(a_i)_{i=1}^m$ establishes that the minimal perimeter inside each linear region intersected with $(-B,B)^d$ is lower bounded. 
As the realization is linear on each region, this implies the existence of a constant $C'_B\in(0,\infty)$, such that for every $\Theta\in\cN^A_N$ it holds that
\begin{equation} \label{eq:W1Linf}
    |\cR(\Theta)|_{\W} \le C'_B  \|\cR(\Theta)\|_{L^{\infty}((-B,B)^d)}.
\end{equation}
Now note that for $\cN^A_N$ we can get the same uniform $(4,1/2)$ inverse stability result w.r.t.\@ $|\cdot|_{\W}$ as in Theorem~\ref{cor:main} by choosing $\pi$ to be the identity in~\eqref{eq:reorder_mult}.
Together with~\eqref{eq:W1Linf} this implies the claim.
\end{proof}
\end{document}

%% file: inverse_tikz.tex
\begin{tikzpicture}

\begin{groupplot}[
     group style = {group size = 2 by 1},
     ymin = -2.5,
     ymax = 2.5,
     xmin = -2.5, 
     xmax = 2.5, 
     zmin = -2.5, 
     zmax = 2.5,
     point meta max=1.5, 
     point meta min=-1.5,
     view/h=15]

\nextgroupplot[
    point meta max=1.,
    point meta min=-1.,
    width=5cm,
    height=5cm]

\addplot3[
    opacity=0.8,
    table/row sep=\\,
    patch,
    patch type=polygon,
    vertex count=4,
    patch table ={
        0 1 2 2\\
        0 2 3 4\\
        0 4 5 5\\
        0 5 6 1\\
    }
]
table {
    x y z\\
    0 0 0\\
    0 2.5 0\\
    2.5 2.5 2.5\\
    2.5 -2.5 -2.5\\
    0 -2.5 -2.5\\
    -2.5 -2.5 0\\
    -2.5 2.5 0\\
};

\addplot3[mark=none,black!50, line width=1pt,opacity=0.5] coordinates {(0,-2.5,-2.5) (0,2.5,-2.5)};
\addplot3[mark=none, black, line width=1pt,opacity=0.5] coordinates {(2.5,2.5,-2.5) (-2.5,-2.5,-2.5)};

\draw[->, black!50, line width=1pt, opacity=0.5] (axis cs:0,0,-2.5) -- (axis cs:1,0,-2.5);
\draw[->, line width=1pt, opacity=0.5] (axis cs:0,0,-2.5) -- (axis cs:1,-1,-2.5);

\nextgroupplot[
    width=5cm,
    height=5cm
]

\addplot3[
    opacity=0.8,
    table/row sep=\\,
    patch,
    patch type=polygon,
    vertex count=5,
    patch table ={
        0 1 2 2 2\\
        0 2 3 4 5\\
        0 5 6 6 6\\
        0 6 7 8 1\\
    }
]
table {
    x y z\\
    0 0 0\\
    0 2.5 0\\
    0.3125 2.5 1.25\\
    2.5 2.5 1.25\\
    2.5 -2.5 -1.25\\
    0 -2.5 -1.25\\
    -0.3125 -2.5 0\\ 
    -2.5 -2.5 0\\
    -2.5 2.5 0\\
};

\addplot3[mark=none,black!50, line width=1pt,opacity=0.5] coordinates {(0,-2.5,-2.5) (0,2.5,-2.5)};
\addplot3[mark=none, black, line width=1pt,opacity=0.5] coordinates {(-0.3125,-2.5,-2.5) (0.3125,2.5,-2.5)};

\draw[->, black!50, line width=1pt, opacity=0.5] (axis cs:0,0,-2.5) -- (axis cs:2,0,-2.5);
\draw[->, line width=1pt, opacity=0.5] (axis cs:0,0,-2.5) -- (axis cs:2,-0.25,-2.5);

\end{groupplot}
\end{tikzpicture}

%% file: bal_tikz.tex
\begin{tikzpicture}

\begin{groupplot}[
     group style = {group size = 2 by 1},
     ymin = -1.5,
     ymax = 1.5,
     xmin = -1.5, 
     xmax = 1.5, 
     zmin = -1.5, 
     zmax = 1.5,
     point meta max=1.5, 
     point meta min=-1.5,
     view/h=15]

\nextgroupplot[
    width=4cm,
    height=4cm]

\addplot3[
    opacity=0.8,
    table/row sep=\\,
    patch,
    patch type=polygon,
    vertex count=4,
    patch table ={
        0 1 2 3\\
        3 4 5 0\\
    }
]
table {
    x y z\\
    0 1.5 0\\
    1.5 1.5 0\\
    1.5 -1.5 0\\
    0 -1.5 0\\
    -1.5 -1.5 0\\
    -1.5 1.5 0\\
};

\addplot3[mark=none,black, line width=1pt,opacity=0.5] coordinates {(0,-1.5,-1.5) (0,1.5,-1.5)};

\draw[->, black, line width=1pt, opacity=0.5] (axis cs:0,0,-1.5) -- (axis cs:0.5,0,-1.5);

\nextgroupplot[
    width=4cm,
    height=4cm]

\addplot3[
    opacity=0.8,
    table/row sep=\\,
    patch,
    patch type=polygon,
    vertex count=4,
    patch table ={
        0 1 4 5\\
        1 2 3 4\\
    }
]
table {
    x y z\\
    1.5 1.5 0.5\\
    1.5 0 0\\
    1.5 -1.5 0\\
    -1.5 -1.5 0\\
    -1.5 0 0\\
    -1.5 1.5 0.5\\
};

\addplot3[mark=none,black, line width=1pt,opacity=0.5] coordinates {(-1.5,0,-1.5) (1.5,0,-1.5)};

\draw[->, black, line width=1pt, opacity=0.5] (axis cs:0,0,-1.5) -- (axis cs:0,1,-1.5);

\end{groupplot}
\end{tikzpicture}

%% file: redundand_tikz.tex
\begin{tikzpicture}

\begin{groupplot}[
     group style = {group size = 2 by 1},
     ymin = -1.5,
     ymax = 1.5,
     xmin = -1.5, 
     xmax = 1.5, 
     zmin = -4.5, 
     zmax = 4.5,
     point meta max=1.5, 
     point meta min=-1.5,
     view/h=15]

\nextgroupplot[
    width=4cm,
    height=4cm]

\addplot3[
    opacity=0.8,
    table/row sep=\\,
    patch,
    patch type=polygon,
    vertex count=4,
    patch table ={
        0 1 2 3\\
        3 4 5 0\\
    }
]
table {
    x y z\\
    0 1.5 0\\
    1.5 1.5 3\\
    1.5 -1.5 3\\
    0 -1.5 0\\
    -1.5 -1.5 0\\
    -1.5 1.5 0\\
};

\addplot3[mark=none,black, line width=1pt,opacity=0.5] coordinates {(0,-1.5,-4.5) (0,1.5,-4.5)};

\draw[->, black, line width=1pt, opacity=0.5] (axis cs:0,0,-4.5) -- (axis cs:1,0,-4.5);

\nextgroupplot[
    width=4cm,
    height=4cm
]

\addplot3[
    opacity=0.8,
    table/row sep=\\,
    patch,
    patch type=polygon,
    vertex count=4,
    patch table ={
        0 1 2 3\\
        0 3 4 5\\
        0 5 6 7\\
        0 7 8 1\\
    }
]
table {
    0 0 0\\
    0 1.5 0.75\\
    1.5 1.5 3.75\\
    1.5 0 3\\
    1.5 -1.5 3\\
    0 -1.5 0\\
    -1.5 -1.5 0\\
    -1.5 0 0\\
    -1.5 1.5 0.75\\
};

\addplot3[mark=none,black, line width=1pt,opacity=0.5] coordinates {(0,-1.5,-4.5) (0,1.5,-4.5)};
\addplot3[mark=none, black!50, line width=1pt,opacity=0.5] coordinates {(-1.5,0,-4.5) (1.5,0,-4.5)};

\draw[->, black, line width=1pt, opacity=0.5] (axis cs:0,0,-4.5) -- (axis cs:1,0,-4.5);
\draw[->, black!50, line width=1pt, opacity=0.5] (axis cs:0,0,-4.5) -- (axis cs:0,1,-4.5);

\end{groupplot}
\end{tikzpicture}

%% file: flip1_tikz.tex
\begin{tikzpicture}

\begin{groupplot}[
     group style = {group size = 2 by 1},
     point meta max=1.5, 
     point meta min=-1.5,
    ymin = -1.5,
    ymax = 1.5,
    xmin = -1.5, 
    xmax = 1.5, 
    zmin = -1.5, 
    zmax = 1.5,
     view/h=15]

\nextgroupplot[
    width=4cm,
    height=4cm,]

\addplot3[
    opacity=0.8,
    table/row sep=\\,
    patch,
    patch type=polygon,
    vertex count=4,
    patch table ={
        0 1 2 3\\
        0 3 4 5\\
        0 5 6 6\\
        0 6 7 8\\
        0 8 9 10\\
        0 10 1 1\\
    }
]
table {
    x y z\\ 
    0 0 0 \\
    0.75 1.5 0\\
    1.5 1.5 0\\
    1.5 0 0\\
    1.5 -1.5 0\\
    0.75 -1.5 0\\
    -0.75 -1.5 0\\
    -1.5 -1.5 0\\
    -1.5 0 0\\
    -1.5 1.5 0\\
    -0.75 1.5 0\\
};

\addplot3[mark=none,black, line width=1pt,opacity=0.5] coordinates {(-1.5,0,-1.5) (1.5,0,-1.5)};

\addplot3[mark=none,black, line width=1pt,opacity=0.5] coordinates {(-0.75,-1.5,-1.5) (0.75,1.5,-1.5)};

\addplot3[mark=none,black, line width=1pt,opacity=0.5] coordinates {(-0.75,1.5,-1.5) (0.75,-1.5,-1.5)};

\draw[->, black, line width=1pt, opacity=0.5] (axis cs:0,0,-1.5) -- (axis cs:0,1,-1.5);

\draw[->, black, line width=1pt, opacity=0.5] (axis cs:0,0,-1.5) -- (axis cs:1,-0.5,-1.5);

\draw[->, black, line width=1pt, opacity=0.5] (axis cs:0,0,-1.5) -- (axis cs:-1,-0.5,-1.5);

\draw[->, black!50, line width=1pt, opacity=0.5] (axis cs:0,0,-1.5) -- (axis cs:0,-1,-1.5);

\draw[->, black!50, line width=1pt, opacity=0.5] (axis cs:0,0,-1.5) -- (axis cs:-1,0.5,-1.5);

\draw[->, black!50, line width=1pt, opacity=0.5] (axis cs:0,0,-1.5) -- (axis cs:1,0.5,-1.5);

\nextgroupplot[
    width=4cm,
    height=4cm]

\addplot3[
    opacity=0.8,
    table/row sep=\\,
    patch,
    patch type=polygon,
    vertex count=4,
    patch table ={
        0 1 2 3\\
        3 4 5 0\\
    }
]
table {
    x y z\\
    0 1.5 0\\
    1.5 1.5 0.5\\
    1.5 -1.5 0.5\\
    0 -1.5 0\\
    -1.5 -1.5 0\\
    -1.5 1.5 0\\
};

\addplot3[mark=none,black, line width=1pt,opacity=0.5] coordinates {(0,-1.5,-1.5) (0,1.5,-1.5)};

\draw[->, black, line width=1pt, opacity=0.5] (axis cs:0,0,-1.5) -- (axis cs:1,0,-1.5);

\end{groupplot}
\end{tikzpicture}

%% file: flip_tikz.tex
\begin{tikzpicture}

\begin{axis}[
     height = 4cm,
     width = 4cm,
     ymin = -2.5,
     ymax = 2.5,
     xmin = -2.5, 
     xmax = 2.5, 
     zmin = -2.5, 
     zmax = 2.5,
     point meta max=1.5, 
     point meta min=-1.5,
     view/h=15]

\draw[->, black!50, line width=1pt] (axis cs:0,0,0) -- (axis cs:2,2,0.5);
\draw[->, black!50, line width=1pt] (axis cs:0,0,0) -- (axis cs:-2,2,0.5);
\draw[->, black!50, line width=1pt] (axis cs:0,0,0) -- (axis cs:0,-2.83,0.35);

\draw[->, black, line width=1pt] (axis cs:0,0,0) -- (axis cs:-2,-2,-0.5);
\draw[->, black, line width=1pt] (axis cs:0,0,0) -- (axis cs:2,-2,-0.5);
\draw[->, black, line width=1pt] (axis cs:0,0,0) -- (axis cs:0,2.83,-0.35);

\end{axis}
\end{tikzpicture}

%% file: local_tikz.tex
\begin{tikzpicture}

\begin{groupplot}[
     group style = {group size = 2 by 1},
     ymin = -1.5,
     ymax = 1.5,
     xmin = -1.5, 
     xmax = 1.5, 
     zmin = -1.5, 
     zmax = 1.5,
     point meta max=1.5, 
     point meta min=-1.5,
     view/h=15]

\nextgroupplot[
    width=5cm,
    height=5cm]

\addplot3[mark=none, dashed, black, line width=0.5pt] coordinates {(0.8, 0.5, -1.5) (0.8, 0.5, 1)};
\addplot3[mark=none, dashed, black, line width=0.5pt] coordinates {(-0.4, 0.6, -1.5) (-0.4, 0.6, 0)};

\addplot3[
    opacity=0.8,
    table/row sep=\\,
    patch,
    patch type=polygon,
    vertex count=4,
    patch table ={
        0 1 2 3\\
        3 4 5 0\\
    }
]
table {
    x y z\\
    0 1.5 0\\
    1.5 1.5 0\\
    1.5 -1.5 0\\
    0 -1.5 0\\
    -1.5 -1.5 0\\
    -1.5 1.5 0\\
};

\draw[->, line width=1pt, opacity=0.5] (axis cs:0,0,-1.5) -- (axis cs:-1,0,-1.5);

\draw[dotted, line width=0.8pt, opacity=0.5]
    (axis cs:0,0,-1.5) node {}
 -- (axis cs:-1,1,-1.5) node {}
 -- (axis cs:1,1,-1.5) node {};
 
\draw[dotted, line width=0.8pt, opacity=0.5]
    (axis cs:0,0,-1.5) node {}
 -- (axis cs:1,-1,-1.5) node {}
 -- (axis cs:1,1,-1.5) node {}
 -- cycle;

\addplot3[mark=none, black, line width=1pt, opacity=0.5] coordinates {(0,-1.5,-1.5) (0,1.5,-1.5)};

\addplot3[scatter, only marks, nodes near coords*={$(x^2,y^2)$}]
coordinates{
(0.8, 0.5, 1) };
\addplot3[scatter, nodes near coords*={$(x^1,y^1)$}]
coordinates{
(-0.4, 0.6, 0) };

\nextgroupplot[
    width=5cm,
    height=5cm]

\addplot3[mark=none, dashed, black, line width=0.5pt] coordinates {(0.8, 0.5, -1.5) (0.8, 0.5, 1)};
\addplot3[mark=none, dashed, black, line width=0.5pt] coordinates {(-0.4, 0.6, -1.5) (-0.4, 0.6, 0)};

\addplot3[
    opacity=0.8,
    table/row sep=\\,
    patch,
    patch type=polygon,
    vertex count=3,
    patch table ={
        0 1 2\\
        2 3 0\\
    }
]
table {
    x y z\\
    1.5 1.5 0\\
    1.5 -1.5 1\\
    -1.5 -1.5 0\\
    -1.5 1.5 0\\
};

\draw[->, line width=1pt, opacity=0.5] (axis cs:0,0,-1.5) -- (axis cs:1,-1,-1.5);

\draw[dotted, line width=0.8pt, opacity=0.5]
    (axis cs:0,0,-1.5) node {}
 -- (axis cs:-1,1,-1.5) node {}
 -- (axis cs:1,1,-1.5) node {};
 
\draw[dotted, line width=0.8pt, opacity=0.5]
    (axis cs:0,0,-1.5) node {}
 -- (axis cs:1,-1,-1.5) node {}
 -- (axis cs:1,1,-1.5) node {}
 -- cycle;

\addplot3[mark=none, black, line width=1pt, opacity=0.5] coordinates {(-1.5,-1.5,-1.5) (1.5,1.5,-1.5)};

\addplot3[scatter, only marks, nodes near coords*={$(x^2,y^2)$}]
coordinates{
(0.8, 0.5, 1) };

\addplot3[scatter, nodes near coords*={$(x^1,y^1)$}]
coordinates{
(-0.4, 0.6, 0) };

\end{groupplot}
\end{tikzpicture}

%% file: proof_tikz.tex
\begin{tikzpicture}
\def\r{3}
\def\o{9}
\def\hs{2}
\def\hr{2.5}
\def\ho{0.5}
\def\hf{1.34375}
\def\sx{4.4}
\def\sy{-1.7}
\draw[thick,->] (0,0,0) -- (0,3.5,0) node[above]{\small$\mathcal{L}$};

\begin{scope}[canvas is plane={O(0,0,0)x(1,0,0)y(0,0,1)},blue!70]
\coordinate (middle) at (\r,0);
\coordinate (opt) at (\o,0);
\coordinate (shift) at (\sx,\sy);
\coordinate (outer) at (0.8786,-2.121);
\draw[fill, opacity=0.05] (5.5,0) ellipse (5.5cm and 4.5cm);
\draw[dashed,opacity=0.5] (middle) circle (\r);
\draw[dashed,opacity=0.5] (middle) circle (\r*0.5);
\node at (10,-1.5) {\small $S$};
\draw[fill] (opt) circle (2pt) node[below right] {\small $g$};
\draw[<->, shorten <=0.1cm, shorten >=0.1cm] (middle) -- node[right,yshift=-6pt, xshift=-8pt] {\small $r'$} (outer);
\draw[fill] (middle)+(\r*0.5,0) circle (2pt) node[below right] {\small $f$};
\draw[fill] (shift) circle (2pt) node[right] {\small$\cR(\Phi)$};
\draw[fill] (middle) circle (2pt) node[below right] {\small$g_*$};
\draw[black, <->, shorten <=0.1cm, shorten >=0.1cm] (middle)+(\r*0.5,0) -- node[right,yshift=-2pt] {\small $\eta$} (shift);
\end{scope}

\begin{scope}[canvas is plane={O(\r,0,0)x(\sx,0,\sy)y(\r,1,0)},orange!70]
\draw[black] (1,0) -- (1,\hr);
\draw (-.5,2.125) parabola bend (0,2) (1.5,3.125);
\fill[opacity=0.1] (0,2) parabola (1,2.5) |- (0,0);
\end{scope}
\begin{scope}[canvas is plane={O(0,0,0)x(1,0,0)y(0,1,0)},red!70]
\draw[black] (3,0) -- (3,2);
\draw[black] (9,0) -- (9,0.5);
\draw[dashed,thick] (\r,\hs) -- (\o,\ho);
\draw[black] (1.5*\r,0) -- (1.5*\r,\hf);
\draw (2.5,2.26042) parabola bend (9,0.5) (9.5,0.51042);
\fill[opacity=0.1] (9,0.5) parabola (3,2) |- (9,0);
\end{scope}

\path (\r,\hs,0) node[circle, fill, inner sep=1]{};
\path (\sx,\hr,\sy) node[circle, fill, inner sep=1]{};
\path (\r,\hs,0) node[circle, fill, inner sep=1]{};
\path (1.5*\r,\hf,0) node[circle, fill, inner sep=1]{};
\path (\o,\ho,0) node[circle, fill, inner sep=1]{};

\draw[thick] (-3pt,\ho) -- (3pt,\ho) node[left=6pt] {\small$\mathcal{L}(g)$};

\draw[thick] (-3pt,\hf) -- (3pt,\hf) node[left=6pt] (f) {\small$\mathcal{L}(f)$};

\draw[thick] (-3pt,\hs) -- (3pt,\hs) node[left=6pt] {\small$\mathcal{L}(g_*)$};

\draw[thick] (-3pt,\hr) -- (3pt,\hr) node[left=6pt] (real) {\small$\mathcal{L}(\cR \Phi)$};

\draw[black, <->, shorten <=0.1cm, shorten >=0.1cm] (6pt,\hf) -- node[right] {\small $c\eta$} (6pt,\hr);

\end{tikzpicture}

%% file: main.bbl
\begin{thebibliography}{40}
\providecommand{\natexlab}[1]{#1}
\providecommand{\url}[1]{\texttt{#1}}
\expandafter\ifx\csname urlstyle\endcsname\relax
  \providecommand{\doi}[1]{doi: #1}\else
  \providecommand{\doi}{doi: \begingroup \urlstyle{rm}\Url}\fi

\bibitem[Allen-Zhu et~al.(2018)Allen-Zhu, Li, and
  Song]{Allen-Zhu2018AOver-Parameterization}
Z.~Allen-Zhu, Y.~Li, and Z.~Song.
\newblock {A Convergence Theory for Deep Learning via Over-Parameterization}.
\newblock \emph{arXiv:1811.03962}, 2018.

\bibitem[Anthony and Bartlett(2009)]{anthony2009}
M.~Anthony and P.~Bartlett.
\newblock \emph{Neural Network Learning: Theoretical Foundations}.
\newblock Cambridge University Press, 2009.

\bibitem[Arora et~al.(2018)Arora, Ge, Neyshabur, and Zhang]{arora2018stronger}
S.~Arora, R.~Ge, B.~Neyshabur, and Y.~Zhang.
\newblock Stronger generalization bounds for deep nets via a compression
  approach.
\newblock In \emph{International Conference on Machine Learning}, pages
  254--263, 2018.

\bibitem[Bansal et~al.(2018)Bansal, Chen, and Wang]{Bansal18}
N.~Bansal, X.~Chen, and Z.~Wang.
\newblock Can we gain more from orthogonality regularizations in training deep
  networks?
\newblock In S.~Bengio, H.~Wallach, H.~Larochelle, K.~Grauman, N.~Cesa-Bianchi,
  and R.~Garnett, editors, \emph{Advances in Neural Information Processing
  Systems 31}, pages 4261--4271. Curran Associates, Inc., 2018.

\bibitem[Bartlett et~al.(2017{\natexlab{a}})Bartlett, Foster, and
  Telgarsky]{BartlettFT17}
P.~L. Bartlett, D.~J. Foster, and M.~Telgarsky.
\newblock Spectrally-normalized margin bounds for neural networks.
\newblock \emph{arXiv:1706.08498}, 2017{\natexlab{a}}.

\bibitem[Bartlett et~al.(2017{\natexlab{b}})Bartlett, Harvey, Liaw, and
  Mehrabian]{Bartlett2017Nearly-tightNetworks}
P.~L. Bartlett, N.~Harvey, C.~Liaw, and A.~Mehrabian.
\newblock {Nearly-tight VC-dimension and pseudodimension bounds for piecewise
  linear neural networks}.
\newblock \emph{arXiv:1703.02930}, 2017{\natexlab{b}}.

\bibitem[Berner et~al.(2018)Berner, Grohs, and
  Jentzen]{Berner2018AnalysisEquations}
J.~Berner, P.~Grohs, and A.~Jentzen.
\newblock {Analysis of the generalization error: Empirical risk minimization
  over deep artificial neural networks overcomes the curse of dimensionality in
  the numerical approximation of Black-Scholes partial differential equations}.
\newblock \emph{arXiv:1809.03062}, 2018.

\bibitem[Berner et~al.(2019)Berner, Elbr{\"a}chter, Grohs, and
  Jentzen]{berner2019towards}
J.~Berner, D.~Elbr{\"a}chter, P.~Grohs, and A.~Jentzen.
\newblock Towards a regularity theory for {ReLU} networks--chain rule and
  global error estimates.
\newblock \emph{arXiv:1905.04992}, 2019.

\bibitem[B{\"o}lcskei et~al.(2017)B{\"o}lcskei, Grohs, Kutyniok, and
  Petersen]{bolcskei2017optimal}
H.~B{\"o}lcskei, P.~Grohs, G.~Kutyniok, and P.~Petersen.
\newblock Optimal approximation with sparsely connected deep neural networks.
\newblock \emph{arXiv:1705.01714}, 2017.

\bibitem[Burger and Neubauer(2001)]{Burger2001235}
M.~Burger and A.~Neubauer.
\newblock {Error Bounds for Approximation with Neural Networks }.
\newblock \emph{Journal of Approximation Theory}, 112\penalty0 (2):\penalty0
  235--250, 2001.

\bibitem[Choromanska et~al.(2015)Choromanska, Henaff, Mathieu, Arous, and
  LeCun]{choromanska2015loss}
A.~Choromanska, M.~Henaff, M.~Mathieu, G.~B. Arous, and Y.~LeCun.
\newblock The loss surfaces of multilayer networks.
\newblock In \emph{Artificial Intelligence and Statistics}, pages 192--204,
  2015.

\bibitem[Czarnecki et~al.(2017)Czarnecki, Osindero, Jaderberg, Swirszcz, and
  Pascanu]{czarnecki2017sobolev}
W.~M. Czarnecki, S.~Osindero, M.~Jaderberg, G.~Swirszcz, and R.~Pascanu.
\newblock Sobolev training for neural networks.
\newblock In \emph{Advances in Neural Information Processing Systems}, pages
  4278--4287, 2017.

\bibitem[Du et~al.(2018)Du, Lee, Li, Wang, and Zhai]{Du2018GradientNetworks}
S.~S. Du, J.~D. Lee, H.~Li, L.~Wang, and X.~Zhai.
\newblock {Gradient Descent Finds Global Minima of Deep Neural Networks}.
\newblock \emph{arXiv:1811.03804}, 2018.

\bibitem[Evans(2010)]{Evans2010PartialEdition}
L.~C. Evans.
\newblock \emph{{Partial Differential Equations (second edition)}}.
\newblock Graduate studies in mathematics. American Mathematical Society, 2010.

\bibitem[Evans and Gariepy(2015)]{Evans2015MeasureEdition}
L.~C. Evans and R.~F. Gariepy.
\newblock \emph{{Measure Theory and Fine Properties of Functions, Revised
  Edition}}.
\newblock Textbooks in Mathematics. CRC Press, 2015.

\bibitem[Funahashi(1989)]{Funahashi1989183}
K.-I. Funahashi.
\newblock {On the approximate realization of continuous mappings by neural
  networks}.
\newblock \emph{Neural Networks}, 2\penalty0 (3):\penalty0 183--192, 1989.

\bibitem[Golowich et~al.(2017)Golowich, Rakhlin, and Shamir]{Golowich17}
N.~Golowich, A.~Rakhlin, and O.~Shamir.
\newblock Size-independent sample complexity of neural networks.
\newblock \emph{arXiv:1712.06541}, 2017.

\bibitem[Goodfellow et~al.(2014)Goodfellow, Vinyals, and
  Saxe]{goodfellow2014qualitatively}
I.~J. Goodfellow, O.~Vinyals, and A.~M. Saxe.
\newblock Qualitatively characterizing neural network optimization problems.
\newblock \emph{arXiv:1412.6544}, 2014.

\bibitem[Gribonval et~al.(2019)Gribonval, Kutyniok, Nielsen, and
  Voigtlaender]{Gribonval2019ApproximationNetworks}
R.~Gribonval, G.~Kutyniok, M.~Nielsen, and F.~Voigtlaender.
\newblock {Approximation spaces of deep neural networks}.
\newblock \emph{arXiv: 1905.01208}, 2019.

\bibitem[G{\"u}hring et~al.(2019)G{\"u}hring, Kutyniok, and
  Petersen]{guhring2019error}
I.~G{\"u}hring, G.~Kutyniok, and P.~Petersen.
\newblock Error bounds for approximations with deep {ReLU} neural networks in
  ${W^{s,p}}$ norms.
\newblock \emph{arXiv:1902.07896}, 2019.

\bibitem[Hinton et~al.(2012)Hinton, Srivastava, Krizhevsky, Sutskever, and
  Salakhutdinov]{hinton2012improving}
G.~E. Hinton, N.~Srivastava, A.~Krizhevsky, I.~Sutskever, and R.~R.
  Salakhutdinov.
\newblock Improving neural networks by preventing co-adaptation of feature
  detectors.
\newblock \emph{arXiv:1207.0580}, 2012.

\bibitem[Kawaguchi(2016)]{kawaguchi2016deep}
K.~Kawaguchi.
\newblock Deep learning without poor local minima.
\newblock In \emph{Advances in neural information processing systems}, pages
  586--594, 2016.

\bibitem[Li et~al.(2018)Li, Xu, Taylor, Studer, and
  Goldstein]{li2018visualizing}
H.~Li, Z.~Xu, G.~Taylor, C.~Studer, and T.~Goldstein.
\newblock Visualizing the loss landscape of neural nets.
\newblock In \emph{Advances in Neural Information Processing Systems}, pages
  6389--6399, 2018.

\bibitem[Li and Liang(2018)]{li2018learning}
Y.~Li and Y.~Liang.
\newblock Learning overparameterized neural networks via stochastic gradient
  descent on structured data.
\newblock In \emph{Advances in Neural Information Processing Systems}, pages
  8157--8166, 2018.

\bibitem[Li and Yuan(2017)]{li2017convergence}
Y.~Li and Y.~Yuan.
\newblock Convergence analysis of two-layer neural networks with {ReLU}
  activation.
\newblock In \emph{Advances in Neural Information Processing Systems}, pages
  597--607, 2017.

\bibitem[Liu et~al.(2018)Liu, Simonyan, and Yang]{liu2018darts}
H.~Liu, K.~Simonyan, and Y.~Yang.
\newblock Darts: Differentiable architecture search.
\newblock \emph{arXiv:1806.09055}, 2018.

\bibitem[Mei et~al.(2018)Mei, Montanari, and Nguyen]{mei2018mean}
S.~Mei, A.~Montanari, and P.-M. Nguyen.
\newblock A mean field view of the landscape of two-layer neural networks.
\newblock \emph{Proceedings of the National Academy of Sciences}, 115\penalty0
  (33):\penalty0 E7665--E7671, 2018.

\bibitem[Miikkulainen et~al.(2019)Miikkulainen, Liang, Meyerson, Rawal, Fink,
  Francon, Raju, Shahrzad, Navruzyan, Duffy, and Hodjat]{MIIKKULAINEN2019293}
R.~Miikkulainen, J.~Liang, E.~Meyerson, A.~Rawal, D.~Fink, O.~Francon, B.~Raju,
  H.~Shahrzad, A.~Navruzyan, N.~Duffy, and B.~Hodjat.
\newblock Chapter 15 - evolving deep neural networks.
\newblock In R.~Kozma, C.~Alippi, Y.~Choe, and F.~C. Morabito, editors,
  \emph{Artificial Intelligence in the Age of Neural Networks and Brain
  Computing}, pages 293 -- 312. Academic Press, 2019.

\bibitem[Neyshabur et~al.(2017)Neyshabur, Bhojanapalli, McAllester, and
  Srebro]{neyshabur2017exploring}
B.~Neyshabur, S.~Bhojanapalli, D.~McAllester, and N.~Srebro.
\newblock Exploring generalization in deep learning.
\newblock In \emph{Advances in Neural Information Processing Systems}, pages
  5947--5956, 2017.

\bibitem[Nguyen and Hein(2017)]{Nguyen17loss}
Q.~Nguyen and M.~Hein.
\newblock The loss surface of deep and wide neural networks.
\newblock In \emph{Proceedings of the 34th International Conference on Machine
  Learning - Volume 70}, ICML'17, pages 2603--2612. JMLR.org, 2017.

\bibitem[Pennington and Bahri(2017)]{Pennington2017loss}
J.~Pennington and Y.~Bahri.
\newblock Geometry of neural network loss surfaces via random matrix theory.
\newblock In \emph{Proceedings of the 34th International Conference on Machine
  Learning - Volume 70}, ICML'17, pages 2798--2806. JMLR.org, 2017.

\bibitem[Perekrestenko et~al.(2018)Perekrestenko, Grohs, Elbr{\"a}chter, and
  B{\"o}lcskei]{perekrestenko2018universal}
D.~Perekrestenko, P.~Grohs, D.~Elbr{\"a}chter, and H.~B{\"o}lcskei.
\newblock The universal approximation power of finite-width deep {ReLU}
  networks.
\newblock \emph{arXiv:1806.01528}, 2018.

\bibitem[{Petersen} and {Voigtlaender}(2017)]{Petersen2017}
P.~{Petersen} and F.~{Voigtlaender}.
\newblock {Optimal approximation of piecewise smooth functions using deep
  {ReLU} neural networks}.
\newblock \emph{arXiv:1709.05289}, 2017.

\bibitem[Petersen et~al.(2018)Petersen, Raslan, and
  Voigtlaender]{Petersen2018TopologicalSize}
P.~Petersen, M.~Raslan, and F.~Voigtlaender.
\newblock {Topological properties of the set of functions generated by neural
  networks of fixed size}.
\newblock \emph{arXiv:1806.08459}, 2018.

\bibitem[Rodr{\'\i}guez et~al.(2016)Rodr{\'\i}guez, Gonzalez, Cucurull,
  Gonfaus, and Roca]{rodriguez2016regularizing}
P.~Rodr{\'\i}guez, J.~Gonzalez, G.~Cucurull, J.~M. Gonfaus, and X.~Roca.
\newblock Regularizing cnns with locally constrained decorrelations.
\newblock \emph{arXiv:1611.01967}, 2016.

\bibitem[Safran and Shamir(2016)]{safran2016quality}
I.~Safran and O.~Shamir.
\newblock On the quality of the initial basin in overspecified neural networks.
\newblock In \emph{International Conference on Machine Learning}, pages
  774--782, 2016.

\bibitem[Shaham et~al.(2018)Shaham, Cloninger, and
  Coifman]{ShaCC2015provableAppDNN}
U.~Shaham, A.~Cloninger, and R.~R. Coifman.
\newblock Provable approximation properties for deep neural networks.
\newblock \emph{Applied and Computational Harmonic Analysis}, 44\penalty0
  (3):\penalty0 537 -- 557, 2018.

\bibitem[Shamir and Zhang(2013)]{Shamir2013}
O.~Shamir and T.~Zhang.
\newblock Stochastic gradient descent for non-smooth optimization: Convergence
  results and optimal averaging schemes.
\newblock In \emph{International Conference on Machine Learning}, pages 71--79,
  2013.

\bibitem[Yarotsky(2017)]{yarotsky2017error}
D.~Yarotsky.
\newblock Error bounds for approximations with deep {ReLU} networks.
\newblock \emph{Neural Networks}, 94:\penalty0 103--114, 2017.

\bibitem[Zoph et~al.(2018)Zoph, Vasudevan, Shlens, and Le]{zoph2018learning}
B.~Zoph, V.~Vasudevan, J.~Shlens, and Q.~V. Le.
\newblock Learning transferable architectures for scalable image recognition.
\newblock In \emph{Proceedings of the IEEE conference on computer vision and
  pattern recognition}, pages 8697--8710, 2018.

\end{thebibliography}
